\documentclass{article}

 \usepackage[preprint]{neurips_2026}

\usepackage[utf8]{inputenc} 
\usepackage[T1]{fontenc}    
\usepackage{booktabs}       
\usepackage{amsfonts}       
\usepackage{nicefrac}       
\usepackage[dvipsnames,svgnames]{xcolor}         
\usepackage[hidelinks]{hyperref}       
\hypersetup{colorlinks=true, linkcolor=red, citecolor=blue}
\usepackage{url}            

\usepackage{amsmath}
\usepackage{amssymb}
\usepackage{amsthm}
\usepackage{csquotes}
\usepackage{bbm}
\usepackage{enumitem}
\usepackage[pdftex]{graphicx}
\usepackage[caption=false,font=footnotesize]{subfig}
\usepackage{multirow}
\usepackage[normalem]{ulem}

\usepackage{tabularx}
\usepackage{wrapfig}
\usepackage{pifont}

\usepackage{algorithm}
\usepackage{algorithmicx}
\usepackage{algpseudocode}
\makeatletter
\def\BState{\State\hskip-\ALG@thistlm}
\makeatother
\algnewcommand\algorithmicforeach{\textbf{for each}}
\algdef{S}[FOR]{ForEach}[1]{\algorithmicforeach\ #1\ \algorithmicdo}

\makeatletter
\algnewcommand{\LineComment}[1]{\Statex \hskip\ALG@thistlm #1}
\makeatother
\floatname{algorithm}{Procedure}

\usepackage{natbib}
\usepackage{pgfplots}
\usepgfplotslibrary{colorbrewer}
\definecolor{color1}{HTML}{0011af} 
\definecolor{color2}{HTML}{FF0000} 
\definecolor{color3}{HTML}{8819a0} 
\definecolor{color4}{HTML}{006400} 
\definecolor{color5}{HTML}{B5651D} 
\definecolor{color6}{HTML}{f9a256} 
\definecolor{color7}{HTML}{2CDECB} 
\definecolor{color8}{HTML}{bf418d} 
\definecolor{color9}{HTML}{991212} 
\pgfplotsset{compat=1.18}

\catcode`,\active

\catcode`\,12

\DeclareMathOperator*{\argmin}{\arg\min}

\newcommand{\gcm}{\normalsize \textcolor{color4}{\ding{51}}}
\newcommand{\rcm}{\normalsize \textcolor{color2}{\ding{55}}}

\usepackage{glossaries}
\let\oldgls=\gls\renewcommand{\gls}[1]{{\hypersetup{linkbordercolor=black, linkcolor=black}\oldgls{#1}}}
\let\oldGls=\Gls\renewcommand{\Gls}[1]{{\hypersetup{linkbordercolor=black, linkcolor=black}\oldGls{#1}}}
\let\oldglspl=\glspl\renewcommand{\glspl}[1]{{\hypersetup{linkbordercolor=black, linkcolor=black}\oldglspl{#1}}}
\newacronym{uq}{UQ}{uncertainty quantification}
\newacronym{cp}{CP}{conformal prediction}
\newacronym{wcp}{WCP}{weighted conformal prediction}
\newacronym{mc}{MC}{marginal coverage}
\newacronym{ccc}{CCC}{calibration-conditional coverage}
\newacronym{fl}{FL}{federated learning}
\newacronym{fcp}{FCP}{federated conformal prediction}
\newacronym{da}{DA}{data agent}
\newacronym{iid}{iid}{independent and identically distributed}
\newacronym{cdf}{CDF}{cumulative distribution function}
\newacronym{fwcp}{FWCP}{federated weighted conformal prediction}
\newacronym{pfwcp}{PFWCP}{personalized federated weighted conformal prediction}
\newacronym{ospfwcp}{osPFWCP}{one-shot personalized federated weighted conformal prediction}

\newtheorem{remark}{Remark}

\newtheorem{thm}{Theorem}

\newtheorem{cor}{Corollary}

\newtheorem{con}{Conjecture}

\newtheorem{lem}{Lemma}
\newenvironment{lemma}{\vspace{-0.2cm}\par\noindent\hrulefill\begin{lem}}{\vspace{-0.2cm}\par\noindent\hrulefill\end{lem}}

\newtheorem{prop}{Proposition}
\newenvironment{proposition}{\vspace{-0.2cm}\par\noindent\hrulefill\begin{prop}}{\vspace{-0.2cm}\par\noindent\hrulefill\end{prop}}

\title{Multi-Agent Conformal Prediction with Personalized Statistical Validity}

\author{%
  Martin V. Vejling \\
  Department of Electronic Systems\\
  Aalborg University,
  Aalborg, Denmark \\
  \texttt{mvv@es.aau.dk} \\
  \And
  Christophe A. N. Biscio \\
  Department of Mathematical Sciences \\
  Aalborg University,
  Aalborg, Denmark \\
  \texttt{christophe@math.aau.dk} \\
  \And
  Adrien~Mazoyer \\
  Institut de Mathématiques de Toulouse\\
  Université de Toulouse,
  Toulouse, France \\
  \texttt{adrien.mazoyer@math.univ-toulouse.fr} \\
  \And
  Petar Popovski \\
  Department of Electronic Systems \\
  Aalborg University,
  Aalborg, Denmark \\
  \texttt{petarp@es.aau.dk} \\
  \And
  Shashi Raj Pandey \\
  Department of Electronic Systems \\
  Aalborg University,
  Aalborg, Denmark \\
  \texttt{srp@es.aau.dk} \\
}

\begin{document}

\maketitle

\begin{abstract}
    Uncertainty quantification is essential in high-stakes machine learning tasks. However, one of the principled solutions, \emph{conformal prediction}, faces challenges under limited local calibration data, privacy constraints, and data heterogeneity. In multi-agent settings, existing works do not simultaneously and satisfactorily address these challenges with guarantees either limited to averages across agents or losing validity in heterogeneous settings. Hence, we propose \textbf{p}ersonalized \textbf{f}ederated \textbf{w}eighted \textbf{c}onformal \textbf{p}rediction (PFWCP), a framework that combines local density ratio weighting with weighted quantile aggregation to correct for heterogeneity while preserving privacy. The method yields asymptotically valid marginal and calibration-conditional coverage guarantees for each participating agent and supports protocols with one-shot communication. Theoretical analysis presents an adjustment to the coverage variance, governed by an effective sample size expression, which is necessary in the context of weighted conformal prediction, and experiments on synthetic and real datasets show improved calibration quality over state-of-the-art federated conformal baselines.
\end{abstract}

\section{Introduction}\label{sec:intro}
\Gls{uq} is a key component of modern machine learning systems, especially in high‑stakes settings where decisions must be accompanied by measures of reliability \citep{Campagner2025:Modeling}. \Gls{cp} offers rigorous, distribution‑free statistical guarantees with improved quality, in terms of tighter prediction sets and more stable coverage behavior, as the amount of quality calibration data increases \citep{Vovk2005:Algorithmic}.

In many realistic applications, a single organization or agent may only have access to a limited amount of local data \citep{Rieke2020:Future}. This makes it difficult to obtain both valid and efficient guarantees using \gls{cp} in isolation. A possible remedy is to pool data from multiple \glspl{da}, however, in such collaborative settings, the data distributions are often heterogeneous across agents, which can lead to systematic violations of the exchangeability assumptions underpinning standard \gls{cp} \citep{Lu2023:Federated}. A principled multi-agent conformal method must account for heterogeneity in how it weights and aggregates calibration information from each agent while being privacy-aware.

Existing \gls{fcp} approaches only partially address this challenge: some methods maintain privacy and support one‑shot communication but assume homogeneous data, leading to severe miscoverage under covariate shift \citep{Humbert2023:One,Humbert2024:Marginal}, while others do not provide personalized coverage guarantees \citep{Lu2023:Federated,Srinivasan2025:FedCF,Wen2026:Efficienct}. Some use \gls{wcp} to correct coverage bias but fail to correctly adjust for the shift in coverage variance, require multiple communication rounds, and typically provide only \gls{mc}, not \gls{ccc}, guarantees \citep{Plassier2023:Conformal,Plassier2024:Efficient}. As a result, none of them simultaneously delivers per‑agent coverage guarantees, data privacy, robustness to heterogeneity, and communication efficiency.

This paper introduces a \gls{pfwcp} framework that jointly corrects for bias and variance induced by data heterogeneity while preserving privacy and requiring only lightweight communication. The key idea is to combine density‑ratio‑based weighting of conformal scores at each agent with a weighted-quantile‑of‑quantiles aggregation at the server, calibrated through an approximation of the per agent \gls{ccc} distribution. This yields prediction sets that (i) provide asymptotically valid \gls{mc} and \gls{ccc} guarantees for each participating agent under covariate shift between agents, and (ii) can be implemented in both standard and one‑shot federated protocols with minimal additional overhead.

\paragraph{Problem formulation:} 
We have an interest in both \gls{mc} and \gls{ccc} guarantees. In the first, for a given significance level $\alpha$, \gls{fcp} constructs a prediction set, denoted $\mathcal{C}_\alpha$, which for a given test covariate $X$ returns a set-valued prediction with guarantee on the \gls{mc}, i.e.,
\begin{equation*}
    \mathbb{P}(Y \in \mathcal{C}_\alpha(X)) \geq 1-\alpha,
\end{equation*}
where $Y$ denotes the true label. Implicitly, the prediction set depends on calibration data which in the case of \gls{fcp} is distributed among the participating \glspl{da}. In the second, for given significance levels $\alpha$ and $\delta$, \gls{fcp} constructs a prediction set, denoted $\mathcal{C}_{\alpha, \delta}$, providing a guarantee on the \gls{ccc}, i.e.,
\begin{equation*}
    \mathbb{P}_\mathcal{D}(\mathbb{P}(Y \in \mathcal{C}_{\alpha,\delta}(X) | \mathcal{D}) \geq 1-\alpha) \geq 1-\delta,
\end{equation*}
where we use $\mathcal{D}$ to denote the calibration data following distribution $\mathbb{P}_\mathcal{D}$. The purpose of this work is to design prediction sets $\mathcal{C}_\alpha(X)$ (or $\mathcal{C}_{\alpha,\delta}(X)$) in a federated setting while being efficient, i.e., minimizing the expected size of the prediction sets, while satisfying the statistical guarantees presented above.

Similarly to \cite{Min2025:Personalized}, we consider a \gls{fl} scenario with $K$ \glspl{da}. The $k$-th \gls{da}, for $k \in [K]=\{1,\dots,K\}$, observes $\Tilde{n}_k$ training data points $\Tilde{\mathcal{D}}^k = \{\Tilde{Z}_i^k\}_{i=1}^{\Tilde{n}_k}$, $\Tilde{Z}_i^k = (\Tilde{X}_i^k, \Tilde{Y}_i^k)$, from unknown distribution $\mathbb{P}^k = \mathbb{P}_{Y|X} \mathbb{P}_X^k$ on the sample space $\mathcal{Z} = \mathcal{X} \times \mathcal{Y}$, and $n_k$ exchangeable calibration data points $\mathcal{D}^k = \{Z_i^k\}_{i=1}^{n_k}$, $Z_i^k = (X_i^k, Y_i^k)$, from the same distribution $\mathbb{P}^k$. We denote by $\mathcal{D} = \bigcup_{k=1}^K\mathcal{D}^k$ the total calibration data, although this data is at no point centralized. For a test point $Z = (X, Y) \sim \mathbb{P} = \mathbb{P}_{Y|X} \mathbb{P}_X$ of unknown test distribution $\mathbb{P}$ where only $X$ is observed, the objective is to provide a prediction set for $Y$ satisfying either \gls{mc} or \gls{ccc} guarantees.

We make the following general assumptions: $(\mathcal{A}1)$ for each k, $Z_1^k, \dots, Z_{n_k}^k$ are exchangeable; $(\mathcal{A}2)$ the data in $\mathcal{D}^k$ is independent of the data in $\mathcal{D}^{l}$ for $k\neq l$; and $(\mathcal{A}3)$ the data in $\Tilde{\mathcal{D}} = \bigcup_{k=1}^K\Tilde{\mathcal{D}}^k$ is independent of the data in $\mathcal{D}$. We consider a practical scenario where the data is heterogeneous due to a covariate shift, i.e., $\mathbb{P}_X^{k} \neq \mathbb{P}_X^{l}$ for $k\neq l$, and, in line with \cite{Min2025:Personalized}, without loss of generality, the test data is observed by \gls{da} $1$, hence, $\mathbb{P} = \mathbb{P}^{1}$. In this sense, the objective is personalization: we strive for a valid and efficient conformal predictor for each individual \gls{da} under data heterogeneity.


\paragraph{Preview of proposed method:}
With the training data, $\Tilde{\mathcal{D}}$, a prediction model is trained, in turn defining a non-conformity score function $\hat{s}:\mathcal{Z}\to\mathbb{R}$, which is evaluated on the calibration data, $\hat{s}_i^k=\hat{s}(X_i^k,Y_i^k)$. In line with \gls{wcp}, as introduced by \cite{Tibshirani2019:Conformal}, weight functions $\hat{\omega}^k:\mathcal{X}\to\mathbb{R}$ are defined as estimators of the density ratios $\omega^k(X) = [{\rm d}\mathbb{P}_X/{\rm d}\mathbb{P}_X^k](X)$ which for $k\in\{2,\dots,K\}$ are fitted using the training data, while $\hat{\omega}^1\equiv 1$, and then evaluated on the calibration data, $\hat{\omega}_i^k = \hat{\omega}(X_i^k)$. We propose to use a weighted-quantile-of-quantiles method, constructing the prediction set as $\mathcal{C}_{\beta,\tau}^{\rm wqq}(X) = \{Y \in \mathcal{Y} : \hat{s}(X,Y)\leq Q_{\beta,\tau}\}$ for
\begin{equation*}
    Q_{\beta, \tau} = {\rm Quantile}\Big(1-\tau; \sum_{k=1}^K w_k \delta_{Q_{\beta}^k}\Big),~\text{where}~~ Q_\beta^k = {\rm Quantile}\Big(1-\beta;\frac{\sum_{i=1}^{n_k} \hat{\omega}_i^k\delta_{\hat{s}_i^k} + \hat{\omega}^k(X)\delta_\infty}{\bar{\omega}^k + \hat{\omega}^k(X)}\Big),
\end{equation*}
where $\bar{\omega}^k = \sum_{j=1}^{n_k} \hat{\omega}_j^k$ is the total calibration data weight, $\beta$ and $\tau$ are the local and aggregation quantiles, respectively, and $w_k$ are normalized weights used to emphasize contributions from \glspl{da} with higher data distribution similarity. To provide statistically valid prediction sets, quantiles $\beta^*$ and $\tau^*$ should be determined such that $\mathcal{C}_{\beta^*,\tau^*}^{\rm wqq}(X)$ satisfies either \gls{mc} or \gls{ccc} guarantees. To enable this, we show in this paper that, as $n_1,\dots,n_K\to\infty$,
\begin{equation*}
    \mathbb{P}\big(Y\in\mathcal{C}_{\beta,\tau}^{\rm wqq}(X)|\mathcal{D}\big)\stackrel{d}{\to} {\rm Quantile}\Big(1-\tau; \sum_{k=1}^K w_k \delta_{U_\beta^k}\Big),
\end{equation*}
where $U_\beta^k \sim {\rm Beta}((1-\beta)(n_{\rm eff}^k + 1), \beta(n_{\rm eff}^k+1))$ for $n_{\rm eff}^k = \Vert [\hat{\omega}_1^k,\dots,\hat{\omega}_{n_k}^k]\Vert_1^2/\Vert [\hat{\omega}_1^k,\dots,\hat{\omega}_{n_k}^k]\Vert_2^2$ which are referred to as the local effective sample sizes.

\paragraph{Related works:}
\setlength{\tabcolsep}{0.4em} 
\begin{table}[t]
    \centering
    \caption{Comparison with closest related works. Shorthands: Federated calibration (Fed.~cali.); Privacy-preserving (Priv.-pres.); Heterogeneous (Het.); Personalized (Pers.); adjusted (adj.).}
    \label{tab:related_works}
    \footnotesize
    \begin{tabular}{@{}l|cccccccc@{}}
        \toprule
                                & Fed.~cali. & One-shot & CCC & Priv.-pres. & Het. & Pers. & Bias adj. & Var adj. \\\midrule
        \cite{Lu2023:Federated} & \gcm & \rcm & \rcm & \gcm & \gcm & \rcm & N/A & N/A \\
        \cite{Plassier2023:Conformal,Plassier2024:Efficient} & \gcm & \rcm & \rcm & \gcm & \gcm & \gcm & \gcm & \rcm \\
        \cite{Humbert2023:One,Humbert2024:Marginal} & \gcm & \gcm & \gcm & \gcm & \rcm & \rcm & N/A & N/A \\
        \cite{Min2025:Personalized} & \rcm & \gcm & \gcm & \gcm & \gcm & \gcm & N/A & N/A \\
        Ours & \gcm & \gcm & \gcm & \gcm & \gcm & \gcm & \gcm & \gcm \\
        \bottomrule
    \end{tabular}
    \vspace{-0.3cm}
\end{table}
\begin{figure}[t]
    \centering
    \begin{minipage}{0.32\linewidth}
        \centering
        \subfloat[\cite{Humbert2023:One,Humbert2024:Marginal}]{\includegraphics[width=1\linewidth, page=3]{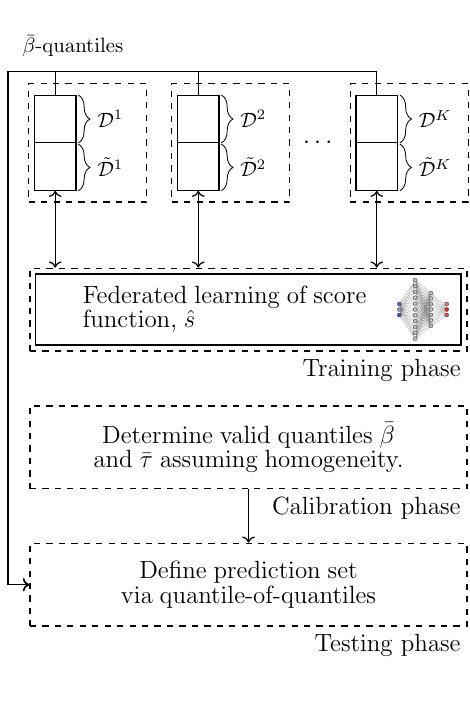}\label{subfig:Humbert}}
    \end{minipage}\hspace{2pt}
    \begin{minipage}{0.32\linewidth}
        \centering
        \subfloat[\cite{Min2025:Personalized}]{\includegraphics[width=1\linewidth, page=4]{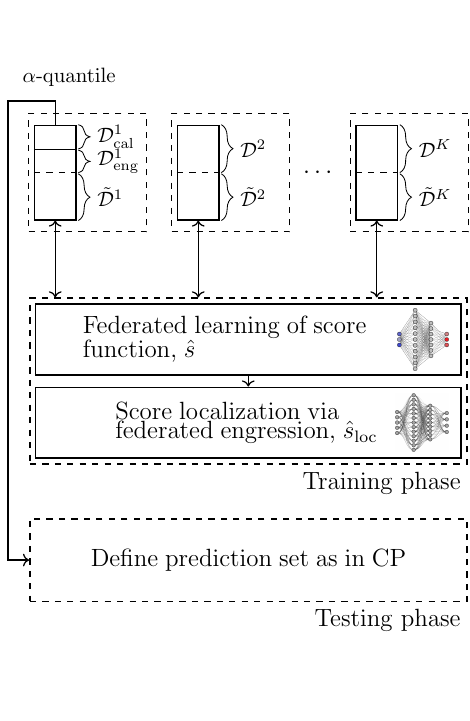}\label{subfig:Min}}
    \end{minipage}\hspace{2pt}
    \begin{minipage}{0.32\linewidth}
        \centering
        \subfloat[Proposed]{\includegraphics[width=1\linewidth, page=5]{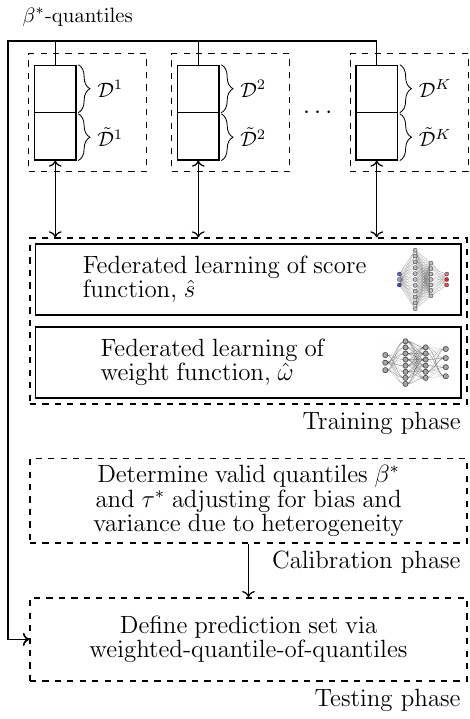}\label{subfig:us}}
    \end{minipage}
    \caption{Outline of the proposed methodology with comparison to the closest existing works.}
    \label{fig:comparison_diagram}
    \vspace{-0.4cm}
\end{figure}
\cite{Lu2023:Federated} introduced a notion of partial exchangeability leading to statistically valid \Gls{fcp} if the test distribution is a mixture of the calibration distributions and addressed privacy of score distribution via sketching. This was extended to conformal fairness by \cite{Srinivasan2025:FedCF} and to group-conditional guarantees in \cite{Wen2026:Efficienct}. Meanwhile, using the idea of \gls{wcp} originally proposed by \cite{Tibshirani2019:Conformal}, \cite{Plassier2023:Conformal} developed \gls{fcp} under label shift, followed by a study for the case of covariate shift in \cite{Plassier2024:Efficient} which is able to correct the coverage bias induced by the covariate shift at the cost of an increased variance. Their approach requires multiple rounds of communication for quantile estimation and provides no \gls{ccc} guarantees. In another work, \cite{Li2026:FCPPro} study federated conformal prediction under label shift and propose a score function building on the idea of adaptive prediction sets including also a local and global prototype similarity score.

A different \gls{fcp} method was proposed in \cite{Humbert2023:One,Humbert2024:Marginal} based on the idea of sharing only a quantile of the scores and then at the central agent taking the quantile-of-quantiles, providing a privacy-preserving solution requiring just one-shot communication, contrary to \gls{fcp} mentioned above. While effective for homogeneous data scenarios, severe issues arise under data heterogeneity. The proposed weighted-quantile-of-quantiles method can be viewed as a generalization of the quantile-of-quantiles approach that is designed to deal with data heterogeneity, and the quantile-of-quantiles is the particularly special case when the density ratios are constantly one, $\hat{\omega}^k(X) \equiv 1$.

The idea of quantile-of-quantiles is closely related to the majority voting scheme of \cite{Cherubin2019:Majority} which notably studies two methods: (i) the overly conservative C-method which bounds the coverage using the Markov inequality, and (ii) the overly liberal I-method that assumes independence of the events $\hat{s}(X,Y) \leq Q_\beta^{k}$. With this perspective, this work proposes an intermediate solution between the overly conservative C-method and overly liberal I-method.

Recently, \cite{Min2025:Personalized} proposed to use decentralized data for localizing the score function employing federated engression and relying only on local data for calibration. This avoids the challenging issue with non-exchangeability in \gls{cp}, however, the efficiency is restricted by the availability of local data.

A comparison of the present work with the existing literature is given in Table~\ref{tab:related_works}. Note that our method aim at solving all the limitations aforementioned. We are using calibration data from all agents (Fed. cali.), preserve privacy (Priv.-pres), require only one-shot communication round (One-shot) and provide \gls{mc} and \gls{ccc} guarantees valid for each individual agent (Pers.) under data heterogeneity (Het.) with appropriate adjustments to account for bias and variance induced by the data heterogeneity. To better see the differences between our approach and the ones proposed by \cite{Humbert2023:One,Humbert2024:Marginal} and \cite{Min2025:Personalized}, we have outlined them in  Figure~\ref{fig:comparison_diagram}.

\section{Personalized federated weighted conformal prediction}\label{sec:method}
We now give a detailed description of the proposed \gls{pfwcp} framework.
We divide the framework into three phases: (i) the training phase where the training data $\Tilde{\mathcal{D}}$ is used to learn the score and weight functions using federated algorithms; (ii) the calibration phase where scores and weights are evaluated on the calibration data $\mathcal{D}$; and (iii) the testing phase where given a test covariate $X$ the conformal prediction set $\mathcal{C}_\alpha(X)$ (or $\mathcal{C}_{\alpha,\delta}(X)$) is constructed.

\textbf{\emph{Training phase:}} we begin by training a prediction model $\hat{f} : \mathcal{X} \to \mathcal{Y}$ on $\Tilde{\mathcal{D}}$ using federated learning and define from this a non-conformity score function $\hat{s} : \mathcal{X} \times \mathcal{Y} \to \mathbb{R}$. Weight functions, defined as estimators of the density ratios $[{\rm d}\mathbb{P}_X/{\rm d}\mathbb{P}_X^k](X)$ for $k\in\{2,\dots,K\}$, are trained on $\Tilde{\mathcal{D}}$ using federated learning and denoted $\hat{\omega}^{k}$, and $\hat{\omega}^1(X) \equiv 1$.


\textbf{\emph{Calibration phase:}} on the calibration data, evaluate the non-conformity scores $\hat{s}_i^k = \hat{s}(X_i^k, Y_i^k)$ and the density ratio weights $\hat{\omega}^k_i = \hat{\omega}^k(X_i^k)$ for $k\in\{2,\dots,K\}$ and $i\in[n_k]$, and compute local effective sample sizes $n_{\rm eff}^k = \Vert [\hat{\omega}_1^k,\dots,\hat{\omega}_{n_k}^k]\Vert_1^2/\Vert [\hat{\omega}_1^k,\dots,\hat{\omega}_{n_k}^k]\Vert_2^2$, and normalized distribution similarity weights $w_k$ for $k\in\{2,\dots,K\}$.
For a local quantile $\beta \in (0, 1)$ and an aggregate quantile $\tau \in [0, 1)$, the prediction set is $\mathcal{C}_{\beta,\tau}^{\rm wqq}(X) = \{Y \in \mathcal{Y} : \hat{s}(X,Y)\leq Q_{\beta,\tau}\}$ for
\begin{equation}\label{eq:wqq_main_paper}
    Q_{\beta, \tau} = {\rm Quantile}\Big(1-\tau; \sum_{k=1}^K w_k \delta_{Q_{\beta}^k}\Big),~\text{where}~~ Q_\beta^k = {\rm Quantile}\Big(1-\beta; \frac{\sum_{i=1}^{n_k} \hat{\omega}_i^k\delta_{\hat{s}_i^k} + \hat{\omega}^k(X)\delta_\infty}{\bar{\omega}^k + \hat{\omega}^k(X)}\Big).
\end{equation}
The core principle of the proposed methodology is that the \gls{ccc} for the $k$-th \gls{da} can be approximated by $U_\beta^k \sim {\rm Beta}((1-\beta)(n_{\rm eff}^k + 1), \beta(n_{\rm eff}^k+1))$. Specifically, according to Proposition~\ref{prop:coverage_approximation}, as $n_1,\dots,n_K\to \infty$,
\begin{equation}\label{eq:cond_cov_approx}
    \mathbb{P}(\hat{s}(X,Y)\leq Q_{\beta, \tau} | \mathcal{D}) = \mathbb{P}(F_{\hat{s}}^{-1}(U)\leq Q_{\beta, \tau} | \mathcal{D}) \stackrel{d}{\to} U_\beta^{(\tau)},
\end{equation}
where $F_{\hat{s}}$ is the \gls{cdf} of the non-conformity scores assumed to be continuous, $U\sim{\rm Uniform}([0,1])$, and  $U_\beta^{(\tau)} = {\rm Quantile}(1-\tau;\sum_{k=1}^K w_k \delta_{U_\beta^k})$.

For \gls{mc}, given significance level $\alpha$, we determine quantile parameters $(\beta^*,\tau^*)$ such that $\mathbb{P}\big(Y \in \mathcal{C}_{\beta^*, \tau^*}^{\rm wqq}(X)\big) \geq 1 - \alpha$ while optimizing the efficiency. This is approximately achieved by minimizing $\mathbb{E}[U_{\beta^*}^{(\tau^*)}]$ under constraint $\mathbb{E}[U_{\beta^*}^{(\tau^*)}] \geq 1-\alpha$, concretely
\begin{equation}\label{eq:choosing_beta_and_tau}
    (\beta^*, \tau^*) = \argmin_{\beta,\tau}~~\{\mathbb{E}[U_\beta^{(\tau)}] : \mathbb{E}[U_\beta^{(\tau)}] \geq 1-\alpha\}.
\end{equation}
This optimization problem does not admit an analytical solution, and in fact, $\mathbb{E}[U_\beta^{(\tau)}]$ has no simple expressive formula. Rather, we approximate $\mathbb{E}[U_\beta^{(\tau)}]$ using Monte Carlo simulations, as outlined in Procedure~\ref{alg:procedure}. We remark here that the choice of the hyperparameter $N_{\rm rep}$ represents an important trade-off between computational complexity and the error made in solving Eq.~\eqref{eq:choosing_beta_and_tau}. A discussion regarding setting $N_{\rm rep}$ is provided in Section~\ref{subsec:implementation}.
\begin{algorithm}[t]
     \caption{Marginal coverage}
     \label{alg:procedure}
     \begin{algorithmic}[1]
        \Statex Input: Significance level $\alpha \in (0,1)$; Effective samples sizes $n_{\rm eff}^1,\dots, n_{\rm eff}^K \leq n$; Grid $\mathcal{G} \subset (0,1)\times[0,1)$; Monte Carlo repetitions $N_{\rm rep} \in \mathbb{N}$.
        \State Compute weights $w_k = n_{\rm eff}^k/(\sum_{k=1}^K n_{\rm eff}^k)$ for $k\in[K]$;
        \ForEach {$(\beta, \tau) \in \mathcal G$}
            \ForEach {$i \in \{1,\dots,N_{\rm rep}\}$}
                \State Sample $U_\beta^k \sim {\rm Beta}\big((1-\beta)(n_{\rm eff}^k + 1), \beta(n_{\rm eff}^k+1)\big)$ for $k\in[K]$;
                \State $U_{\beta,i}^{(\tau)} \gets {\rm Quantile}(1-\tau;\sum_{k=1}^K w_k \delta_{U_\beta^k})$;
            \EndFor
            \State ${\rm Cov}_{\beta, \tau} \gets \frac{1}{N_{\rm rep}} \sum_{i=1}^{N_{\rm rep}} U_{\beta,i}^{(\tau)}$;
        \EndFor
        \State Find $(\beta^*, \tau^*) \gets \argmin_{(\beta,\tau)\in\mathcal{G}}~~\{{\rm Cov}_{\beta, \tau} : {\rm Cov}_{\beta, \tau} \geq 1-\alpha\}$;
        \Statex Output: $(\beta^*, \tau^*)$.
     \end{algorithmic}
\end{algorithm}

\Gls{ccc} can be achieved using similar methodology as for \gls{mc}. The objective is here to determine hyperparameters $\beta_c^*$ and $\tau_c^*$ such that the \gls{ccc} is minimized while maintaining
\begin{equation*}
    \mathbb{P}_\mathcal{D}\big(\mathbb{P}(Y \in \mathcal{C}_{\beta_c^*, \tau_c^*}^{\rm wqq}(X) | \mathcal{D}) \geq 1 - \alpha \big) \geq 1-\delta,
\end{equation*}
for given significance levels $\alpha$ and $\delta$. We propose to approximate the left-side of the above inequality by $\mathbb{P}(U_{\beta}^{(\tau)} \geq 1-\alpha)$ which can be numerically approximated through Monte Carlo simulations. Concretely, we replace line 7 in Procedure~\ref{alg:procedure} by ${\rm CCov}_{\beta, \tau} \gets \frac{1}{N_{\rm rep}}\sum_{i=1}^{N_{\rm rep}} \mathbbm{1}[U_{\beta,i}^{(\tau)} \geq 1-\alpha]$ and line 9 by $(\beta_c^*, \tau_c^*) \gets \argmin_{(\beta,\tau)\in\mathcal{G}}~~\{{\rm CCov}_{\beta, \tau} : {\rm CCov}_{\beta, \tau} \geq 1-\delta\}$.

\textbf{\emph{Testing phase:}} \gls{da} 1 shares the test weight $\hat{\omega}(X)$ which in turn enables the computation of the quantiles $Q_{\beta^*}^k$ (or $Q_{\beta^*_{\rm c}}^k$). These quantiles are shared with the central server that computes the weighted-quantile-of-quantiles $Q_{\beta^*, \tau^*}$ (or $Q_{\beta^*_{\rm c}, \tau^*_{\rm c}}$) and shares it with \gls{da} 1. Finally, the conformal prediction set is constructed as $\mathcal{C}^{\rm wqq}_{\beta^*,\tau^*}(X)$ (or $\mathcal{C}^{\rm wqq}_{\beta^*_c,\tau^*_c}(X)$).

\begin{remark}[Designing the weights]
    In line with the technique of \cite{Tibshirani2019:Conformal}, we train $K-1$ binary classifiers, $\hat{g}_k : \mathcal{X} \to (0, 1)$, using federated learning with training data $\{(\Tilde{X}_i^k, 0)\}_{i=1}^{\Tilde{n}_k} \cup \{(\Tilde{X}_i^1, 1)\}_{i=1}^{\Tilde{n}_1}$. Then, $\hat{\omega}^k(X) = \hat{g}_k(X)/(1-\hat{g}_k(X))$ where $\hat{g}_k(X)$ denotes the estimated probability that $X$ belongs to class $1$ for $k\in\{2,\dots,K\}$. A detailed discussion motivating this construction as well as potential alternatives is presented in Section~\ref{sec:design_weights}.

    We define aggregation weights $w_k \propto n_{\rm eff}^k$. This choice of aggregation weights is natural as it imposes no further computation or communication overhead, and $n_{\rm eff}^k$ can be interpreted as a type of distribution similarity measure: a large $n_{\rm eff}^k$ is indicative of a high distribution similarity, and vice versa. A discussion of alternatives is given in Section~\ref{sec:design_weights}.
\end{remark}

\paragraph{One-shot personalized federated weighted conformal prediction (osPFWCP)}
The \gls{pfwcp} method requires each participating \gls{da} to share a quantile for each arriving test point since the quantiles in \gls{wcp} changes depending on the test point covariate through the weight $\hat{\omega}^k(X)$. We present now a procedure allowing the participating \glspl{da} to share only a single quantile in the calibration phase (prior to testing), and then simply adjust the outer quantile, $\tau$, depending on the arriving test covariate, in order to provide the required statistical guarantees. Thereby, the method requires only minimal communication signaling during the testing phase, however, at the cost of less flexibility in choosing the inner quantile, $\beta$.

Fix an inner quantile, $\breve{\beta}$, and construct the prediction set as $\breve{\mathcal{C}}_{\breve{\beta},\tau}^{\rm wqq}(X) = \{Y \in \mathcal{Y} : \hat{s}(X,Y)\leq \breve{Q}_{\breve{\beta}, \tau}\}$ for
\begin{equation}\label{eq:os_wqq_main}
    \breve{Q}_{\breve{\beta}, \tau} = {\rm Quantile}\Big(1-\tau; \sum_{k=1}^K w_k \delta_{\breve{Q}_{\breve{\beta}}^k}\Big),~\text{where}~~ \breve{Q}_{\breve{\beta}}^k = {\rm Quantile}\Big(1-\breve{\beta};\frac{\sum_{i=1}^{n_k} \hat{\omega}_i^k\delta_{\hat{s}_i^k}}{\bar{\omega}^k}\Big).
\end{equation}
By doing this, the quantiles, $\breve{Q}_{\breve{\beta}}^k$, are computed and shared with the central server only once during the calibration phase, along with added information $\bar{\omega}^k$ and $\bar{\bar{\omega}}^k = \sum_{i=1}^{n_k} \hat{\omega}_i^k \mathbbm{1}[\hat{s}_i^k > \breve{Q}^k_{\breve{\beta}}]$. At the central server, given the test covariate, $X$, $\beta^k$-quantiles such that $\breve{Q}^k_{\breve{\beta}} = Q^k_{\beta^k}$ are derived as $\beta^k = (\hat{\omega}^k(X) + \bar{\bar{\omega}}^k)/(\bar{\omega}^k + \hat{\omega}^k(X))$. Finally, the central server finds a quantile $\breve{\tau}^*$ to realize the required statistical guarantee following an algorithm similar to Procedure~\ref{alg:procedure}, but limiting the parameter search to the outer quantile parameter, $\tau$, and using the inner quantiles $\beta^1,\dots,\beta^K$. The complexity of this is relatively low since it is only a one-dimensional parameter search.

An overview of the communication signaling required to execute either the \gls{pfwcp} or \gls{ospfwcp} algorithms is given in Figure~\ref{fig:communication_diagrams}, illustrating clearly the benefits of the one-shot approach by noticing that minimal communication signaling is required during the testing phase.
\begin{figure}[t]
    \centering
    \begin{minipage}{0.49\linewidth}
        \centering
        \subfloat[PFWCP.]{\includegraphics[width=1\linewidth, page=8]{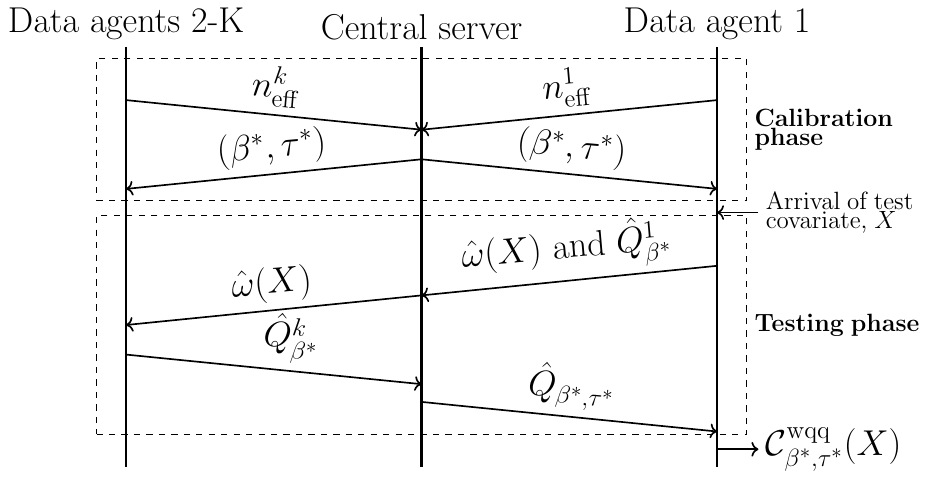}\label{subfig:comms_diagram_proposed}}
    \end{minipage}\hspace{2pt}
    \begin{minipage}{0.49\linewidth}
        \centering
        \subfloat[osPFWCP.]{\includegraphics[width=1\linewidth, page=9]{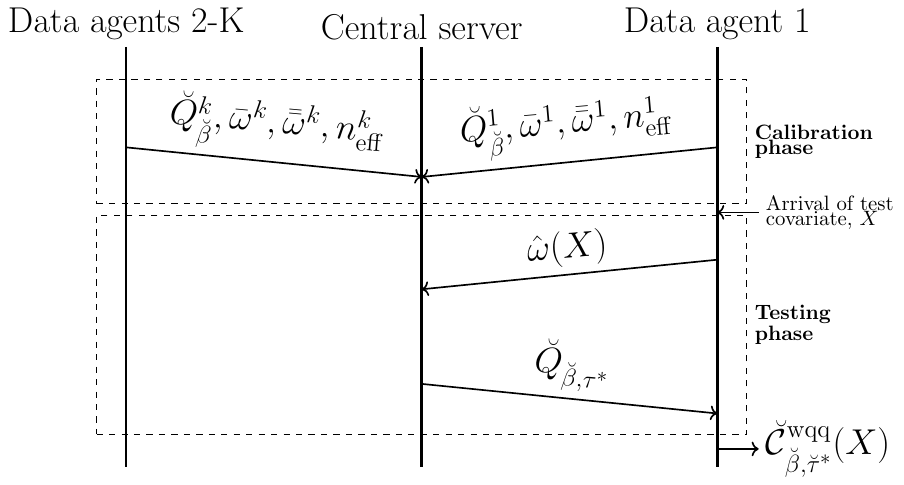}\label{subfig:comms_diagram_oneshot}}
    \end{minipage}
    \caption{Diagrams illustrating the communication signaling required to execute the PFWCP and osPFWCP algorithms.}
    \label{fig:communication_diagrams}
    \vspace{-0.4cm}
\end{figure}

\paragraph{Limitations:}
The scope of the present study is limited to dealing with covariate shifts, and it will be interesting to extend the methodology to more general types of distribution shifts. Further, the presented methodology has some limitations. Namely, since the methods rely on estimating the density ratio, if this is poorly approximated, it will result in failure to properly adjust for the bias and variance in the coverage due to the data heterogeneity. Such a risk is, however, unavoidable when calibrating in the presence of unknown distribution shifts, hence, the limitation is shared with existing works \citep{Plassier2024:Efficient}. Finally, some care has to be taken in the selection of the $\beta$ and $\tau$ quantiles, mainly, it is crucial that $\sum_{i=1}^{n_k} \hat{\omega}_i^{k} \geq (1-\beta)(\bar{\omega}^k+\hat{\omega}^k(X))$ for all $k\in[K]$, otherwise $Q_{\beta}^k$ will not be finite. This can be dealt with by limiting the search space over $\beta$ and $\tau$ to moderate values, see also the discussion in Section~\ref{subsec:implementation}.

\paragraph{Key theoretical result:}
The key result motivating the proposed methodology is now rigorously presented. The proof is inspired by \cite{Gretton2009:Dataset} and is presented in detail in Section~\ref{sec:proofs}.
\begin{proposition}\label{prop:coverage_approximation}
    Assume that the score function is continuously distributed, that $\hat{\omega}^k \equiv \omega^k$, that $\mathbb{P}$ is absolutely continuous with respect to $\mathbb{P}^k$ for $k\in[K]$, and that there exists bounded constants $0<A\leq1\leq B$ such that $\mathbb{P}_X^k(\omega(X_i^k) \in [A,B]) = 1$ for $X_i^k\sim\mathbb{P}_X^{k}$, $k\in[K]$, $i\in[n_k]$.
    Then, as $n_1,\dots,n_K \to \infty$,
    \begin{equation*}
        \mathbb{P}\big(Y\in\mathcal{C}_{\beta,\tau}^{\rm wqq}(X)|\mathcal{D}\big)\stackrel{d}{\to} {\rm Quantile}\Big(1-\tau; \sum_{k=1}^K w_k \delta_{U_\beta^k}\Big),
    \end{equation*}
    where $U_\beta^k \sim {\rm Beta}((1-\beta)(n_{\rm eff}^k + 1), \beta(n_{\rm eff}^k+1))$ for $n_{\rm eff}^k = \Vert [\omega_1^k,\dots,\omega_{n_k}^k]\Vert_1^2/\Vert [\omega_1^k,\dots,\omega_{n_k}^k]\Vert_2^2$.
\end{proposition}

This result directly supports the core argument from Eq.~\eqref{eq:cond_cov_approx}, thereby providing the theoretical justification for Procedure~\ref{alg:procedure}. Notably, the approximation of Eq.~\eqref{eq:cond_cov_approx} relies on accurate estimation of the density ratios and a sufficiently large calibration data size: as the proof indicates, the error of the approximation decays as the calibration data sizes increase. The result is complementary to that of \cite{Pournaderi2026:Training} which also dealt with \gls{ccc} of \gls{wcp} considering the problem through concentration inequalities. Applying their result is, however, not practically straightforward in our setting due to the presence of unknown constants in the expressed \gls{ccc} bounds.

We conducted a numerical experiment on the Tennessee's student teacher achievement ratio dataset \citep{Achilles2008:Tennessee}, simulating data heterogeneity with exponential tilting as in \cite{Tibshirani2019:Conformal} (details on the data setup are found in Section~\ref{subsec:details_figure1}). The result is shown in Figure~\ref{fig:star_illustration_cdf}.

The quantile-of-quantiles method of \cite{Humbert2023:One,Humbert2024:Marginal} approximates the solid blue line with the dashed red line resulting in severe miscoverage. On the other hand, \cite{Plassier2023:Conformal,Plassier2024:Efficient} approximates the solid purple line with the dashed red line, while we propose to approximate it with the dashed purple line. Figure~\ref{fig:star_illustration_cdf} clearly illustrates the benefits of the proposed approximation as compared to the alternatives from the literature.

\section{Numerical experiments}\label{sec:numerical}
In this section, we report results from numerical experiments that study the performance in terms of miscoverage and efficiency of the proposed conformal predictors with those of relevant benchmarks.

The experiments cover nine datasets: two synthetic regression datasets (\emph{Gaussian} and \emph{Poisson}), six real regression datasets (\emph{airfoil}, \emph{bike}, \emph{concrete}, \emph{crime}, \emph{star}, and \emph{protein}) also considered in \cite{Tibshirani2019:Conformal,Humbert2023:One,Humbert2024:Marginal,Min2025:Personalized}, and a real classification dataset (\emph{cifar10c}). Details regarding the datasets and how data heterogeneity is generated are found in Section~\ref{subsec:datasets}. The prediction model is defined as a fully connected neural network, and likewise for the density ratio estimators.
The predictors and density ratio estimators are trained in a federated learning setting and each of the $K=11$ \glspl{da} has $\Tilde{n}_k = 50$ training data points. For calibration, each \gls{da} has $n_k=100$ additional data points. Details of model specifications and training are found in Section~\ref{subsec:implementation}. For the regression tasks, we use as non-conformity score $\hat{s}(X,Y) = |Y-\hat{f}(X)|$ where $\hat{f}$ is the trained regression model, while for the classification task, we use $\hat{s}(X,Y) = 1-[\hat{f}(X)]_Y$ where $[\hat{f}(X)]_Y$ denotes the softmax output of the trained classifier $\hat{f}$ associated with the label $Y$.

As benchmarks we include vanilla \gls{cp} (\emph{CP}), federated generalized localization conformal prediction as in \cite{Min2025:Personalized} (\emph{FGLCP}), quantile-of-quantiles as in \cite{Humbert2023:One,Humbert2024:Marginal} (\emph{FCP-QQ}), and federated \gls{wcp} \cite{Plassier2024:Efficient} (\emph{FWCP}). Variants of our proposal are included, herein, federated \gls{wcp} by quantile-of-quantiles (\emph{FWCP-QQ}), personalized federated \gls{wcp} (\emph{PFWCP}), one-shot personalized federated \gls{wcp} (\emph{osPFWCP}), and one-shot personalized federated \gls{wcp} with federated generalized localization (\emph{osPFWGLCP}). Details of benchmarks are given in Section~\ref{subsec:benchmarks}.
\begin{figure}[t]
    \centering
    \includegraphics[width=0.55\linewidth, page=6]{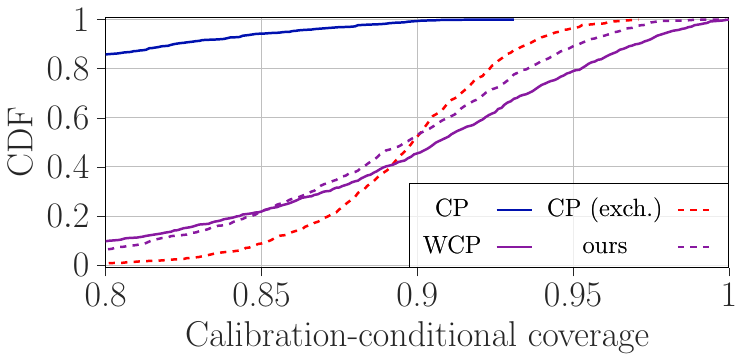}
    \caption{Empirical CDF of the CCC for CP with covariate shift (blue solid line), CP without covariate shift (red dashed line), WCP (purple solid line), and CP without covariate shift but samples limited to the effective sample size (purple dashed line). Significance level $\alpha=0.1$ and calibration data size $n=100$.}
    \label{fig:star_illustration_cdf}
    \vspace{-0.2cm}
\end{figure}
As performance metrics, for a conformal predictor denoted $\mathcal{C}$, we use the \gls{mc} and \gls{ccc} when relevant, and further express the conditional miscoverage (CMC), defined as $\mathbb{E}\big[\big|\mathbb{P}(Y \in \mathcal{C} | \mathcal{D}) - (1-\alpha)\big|\big]$, and the efficiency (Eff.), defined as $\mathbb{E}[|\mathcal{C}|]$. These quantities are estimated through Monte Carlo simulations covering $500$ calibration data simulations with $500$ test data points for each.



\paragraph{Marginal coverage:}
We display in Figure~\ref{fig:boxplots}, boxplots of the coverage and efficiency with all the benchmark methodologies providing \gls{mc} guarantees, on the \emph{star} dataset, when setting the significance level to $\alpha=0.1$. This result shows the high variability in coverage and prediction set size of \emph{CP} and \emph{FGLCP} due to only using local data for calibration, and also highlights the miscoverage with the \emph{FCP-QQ} and \emph{FWCP-QQ} methods due to failure to correct for the bias and variance shifts in the coverage caused by the data heterogeneity. Meanwhile, the proposed methods has only slight miscoverage, with low variability in coverage and prediction set size, and is generally efficient relative to the benchmarks. The most competitive benchmark for providing \gls{mc} guarantees is the \emph{FWCP}.
\begin{figure}[t]
    \centering
    \begin{minipage}{0.49\linewidth}
        \centering
        \subfloat[Coverage.]{\includegraphics[width=1\linewidth, page=10]{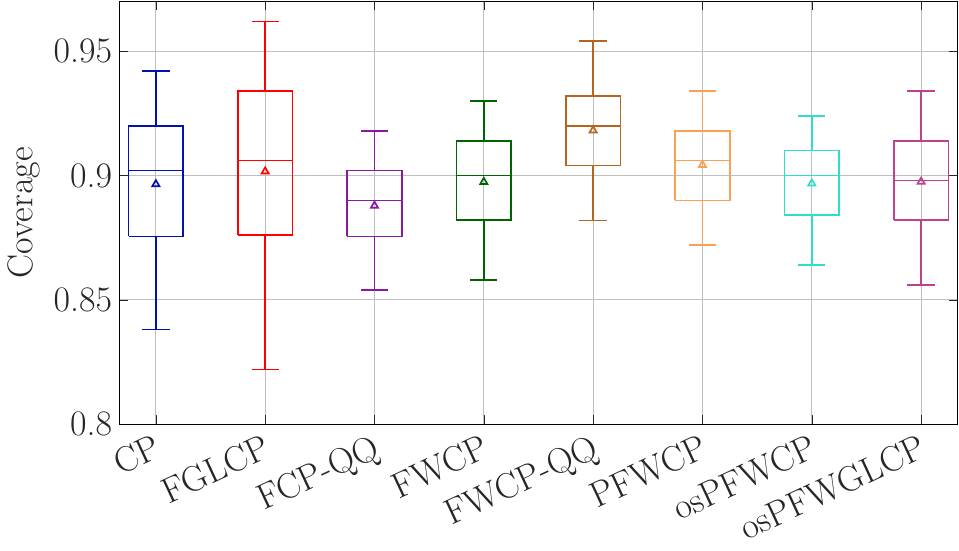}\label{subfig:boxplot_coverage}}
    \end{minipage}\hspace{2pt}
    \begin{minipage}{0.49\linewidth}
        \centering
        \subfloat[Efficiency.]{\includegraphics[width=1\linewidth, page=11]{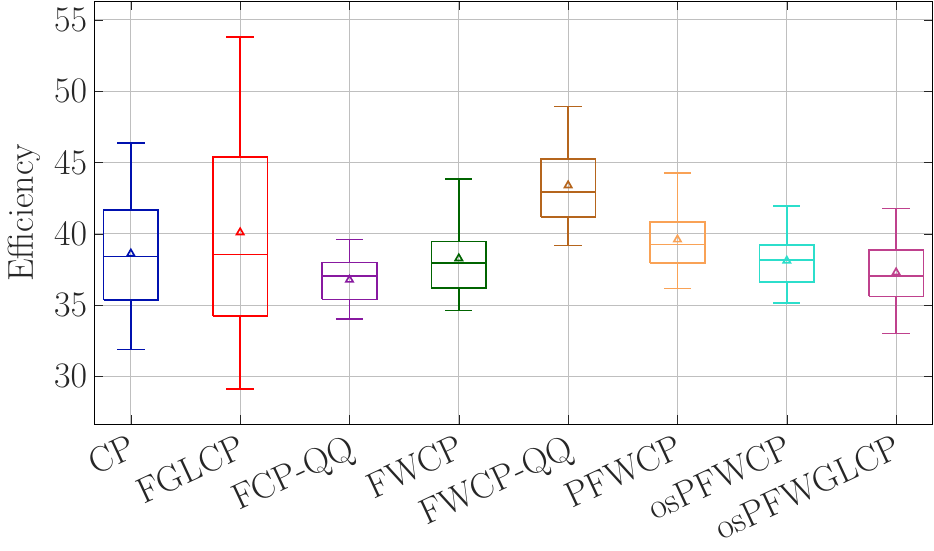}\label{subfig:boxplot_efficiency}}
    \end{minipage}
    \caption{Boxplots of coverage and efficiency for the proposed methods and benchmarks on the \emph{star} dataset.}
    \label{fig:boxplots}
    \vspace{-0.1cm}
\end{figure}
\setlength{\tabcolsep}{0.3em} 
\begin{table}[t]
    \centering
    \caption{Marginal coverage (MC), conditional miscoverage (CMC), and efficiency (Eff.). Methods with MC above $1-\alpha=0.9$, above $0.89$, and below $0.89$ are marked by colors \textcolor{color4}{green}, \textcolor{color6}{yellow}, and \textcolor{color2}{red}, respectively. The best performing method in each dataset, that satisfy the constraint of MC above $0.9$, is highlighted with boldface.}
    \label{tab:reg_marg_main}
    \footnotesize
    \begin{tabular}{llrrrrrrrrr}
    \toprule
                  &                    & Gaussian                    & Poisson                     & airfoil                     & protein                     & bike                        & concrete                    & crime                       & star                        & cifar10c                     \\
    \midrule
    MC (\%)   & CP                 & \textcolor{color6}{$89.57$} & \textcolor{color6}{$89.77$} & \textcolor{color6}{$89.49$} & \textcolor{color6}{$89.48$} & \textcolor{color6}{$89.42$} & \textcolor{color6}{$89.65$} & \textcolor{color6}{$89.58$} & \textcolor{color6}{$89.65$} & \textcolor{color6}{$89.16$} \\
                & FGLCP {\tiny \citep{Min2025:Personalized}}          & \textcolor{color4}{$90.38$} & \textcolor{color4}{$90.28$} & \textcolor{color6}{$89.80$} & \textcolor{color4}{$90.19$} & \textcolor{color4}{$\textbf{90.23}$} & \textcolor{color4}{$\textbf{90.34}$} & \textcolor{color6}{$89.99$} & \textcolor{color4}{$\textbf{90.16}$} & \textcolor{color6}{$89.96$} \\
                & FCP-QQ {\tiny \citep{Humbert2023:One}}         & \textcolor{color2}{$88.70$} & \textcolor{color6}{$89.97$} & \textcolor{color2}{$86.48$} & \textcolor{color4}{$90.20$} & \textcolor{color2}{$88.20$} & \textcolor{color2}{$84.17$} & \textcolor{color6}{$89.11$} & \textcolor{color2}{$88.78$} & \textcolor{color6}{$89.40$} \\
                & FWCP {\tiny \citep{Plassier2024:Efficient}}        & \textcolor{color4}{$90.22$} & \textcolor{color6}{$89.81$} & \textcolor{color4}{$92.32$} & \textcolor{color4}{$90.51$} & \textcolor{color6}{$89.57$} & \textcolor{color6}{$89.93$} & \textcolor{color6}{$89.11$} & \textcolor{color6}{$89.74$} & \textcolor{color2}{$85.23$} \\
                & FWCP-QQ          & \textcolor{color4}{$92.13$} & \textcolor{color4}{$92.34$} & \textcolor{color4}{$94.06$} & \textcolor{color4}{$91.09$} & \textcolor{color4}{$91.10$} & \textcolor{color4}{$91.78$} & \textcolor{color4}{$92.11$} & \textcolor{color4}{$91.82$} & \textcolor{color2}{$85.87$} \\
                & PFWCP     & \textcolor{color4}{$92.01$} & \textcolor{color4}{$91.26$} & \textcolor{color4}{$91.73$} & \textcolor{color4}{$90.67$} & \textcolor{color4}{$90.31$} & \textcolor{color4}{$90.78$} & \textcolor{color4}{$\textbf{90.47}$} & \textcolor{color4}{$90.43$} & \textcolor{color4}{$\textbf{90.32}$} \\
                & osPFWCP   & \textcolor{color6}{$89.07$} & \textcolor{color4}{$\textbf{90.06}$} & \textcolor{color4}{$\textbf{90.12}$} & \textcolor{color4}{$90.28$} & \textcolor{color2}{$88.16$} & \textcolor{color6}{$89.42$} & \textcolor{color6}{$89.00$} & \textcolor{color6}{$89.68$} & \textcolor{color2}{$88.39$} \\
                & osPFWGLCP & \textcolor{color4}{$\textbf{90.19}$} & \textcolor{color4}{$90.17$} & \textcolor{color6}{$89.66$} & \textcolor{color4}{$\textbf{90.11}$} & \textcolor{color6}{$89.21$} & \textcolor{color6}{$89.65$} & \textcolor{color6}{$89.27$} & \textcolor{color6}{$89.75$} & \textcolor{color2}{$85.26$} \\
    \midrule
    CMC (\%)  & CP                 & \textcolor{color6}{$2.67$}  & \textcolor{color6}{$2.54$}  & \textcolor{color6}{$2.64$}  & \textcolor{color6}{$2.70$}  & \textcolor{color6}{$2.60$}  & \textcolor{color6}{$2.71$}  & \textcolor{color6}{$2.46$}  & \textcolor{color6}{$2.64$}  & \textcolor{color6}{$2.67$}  \\
                  & FGLCP {\tiny \citep{Min2025:Personalized}}             & \textcolor{color4}{$3.36$}  & \textcolor{color4}{$3.45$}  & \textcolor{color6}{$3.58$}  & \textcolor{color4}{$3.51$}  & \textcolor{color4}{$3.25$}  & \textcolor{color4}{$3.44$}  & \textcolor{color6}{$3.36$}  & \textcolor{color4}{$3.47$}  & \textcolor{color6}{$3.52$}  \\
                  & FedCP-QQ {\tiny \citep{Humbert2023:One}}          & \textcolor{color2}{$2.54$}  & \textcolor{color6}{$1.41$}  & \textcolor{color2}{$3.57$}  & \textcolor{color4}{$1.44$}  & \textcolor{color2}{$2.17$}  & \textcolor{color2}{$5.83$}  & \textcolor{color6}{$1.57$}  & \textcolor{color2}{$1.82$}  & \textcolor{color6}{$2.54$}  \\
                  & FWCP {\tiny \citep{Plassier2024:Efficient}}              & \textcolor{color4}{$2.37$}  & \textcolor{color6}{$1.66$}  & \textcolor{color4}{$2.51$}  & \textcolor{color4}{$\textbf{1.32}$}  & \textcolor{color6}{$2.04$}  & \textcolor{color6}{$1.58$}  & \textcolor{color6}{$2.94$}  & \textcolor{color6}{$1.76$}  & \textcolor{color2}{$6.52$}  \\
                  & FWCP-QQ          & \textcolor{color4}{$2.53$}  & \textcolor{color4}{$2.61$}  & \textcolor{color4}{$4.09$}  & \textcolor{color4}{$1.62$}  & \textcolor{color4}{$1.89$}  & \textcolor{color4}{$2.21$}  & \textcolor{color4}{$2.67$}  & \textcolor{color4}{$2.33$}  & \textcolor{color2}{$6.30$}  \\
                  & PFWCP    & \textcolor{color4}{$2.68$}  & \textcolor{color4}{$1.94$}  & \textcolor{color4}{$2.15$}  & \textcolor{color4}{$1.56$}  & \textcolor{color4}{$\textbf{1.79}$}  & \textcolor{color4}{$\textbf{1.85}$}  & \textcolor{color4}{$\textbf{1.95}$}  & \textcolor{color4}{$\textbf{1.66}$}  & \textcolor{color4}{$\textbf{2.89}$}  \\
                  & osPFWCP   & \textcolor{color6}{$2.11$}  & \textcolor{color4}{$\textbf{1.57}$}  & \textcolor{color4}{$\textbf{1.64}$}  & \textcolor{color4}{$1.42$}  & \textcolor{color2}{$2.25$}  & \textcolor{color6}{$1.75$}  & \textcolor{color6}{$1.92$}  & \textcolor{color6}{$1.54$}  & \textcolor{color2}{$2.76$}  \\
                  & osPFWGLCP & \textcolor{color4}{$\textbf{2.11}$}  & \textcolor{color4}{$1.75$}  & \textcolor{color6}{$1.82$}  & \textcolor{color4}{$1.53$}  & \textcolor{color6}{$1.84$}  & \textcolor{color6}{$2.16$}  & \textcolor{color6}{$2.10$}  & \textcolor{color6}{$1.82$}  & \textcolor{color2}{$5.14$}  \\
    \midrule
    Eff.          & CP                 & \textcolor{color6}{$0.09$}  & \textcolor{color6}{$9.84$}  & \textcolor{color6}{$6.52$}  & \textcolor{color6}{$2.25$}  & \textcolor{color6}{$2.04$}  & \textcolor{color6}{$20.75$} & \textcolor{color6}{$2.12$}  & \textcolor{color6}{$38.61$} & \textcolor{color6}{$8.94$}  \\
                  & FGLCP {\tiny \citep{Min2025:Personalized}}            & \textcolor{color4}{$0.09$}  & \textcolor{color4}{$15.96$} & \textcolor{color6}{$6.73$}  & \textcolor{color4}{$2.36$}  & \textcolor{color4}{$2.45$}  & \textcolor{color4}{$23.14$} & \textcolor{color6}{$2.29$}  & \textcolor{color4}{$40.10$} & \textcolor{color6}{$9.00$}  \\
                  & FedCP-QQ {\tiny \citep{Humbert2023:One}}          & \textcolor{color2}{$0.09$}  & \textcolor{color6}{$9.75$}  & \textcolor{color2}{$5.47$}  & \textcolor{color4}{$2.28$}  & \textcolor{color2}{$1.92$}  & \textcolor{color2}{$15.11$} & \textcolor{color6}{$2.06$}  & \textcolor{color2}{$36.77$} & \textcolor{color6}{$8.96$}  \\
                  & FWCP {\tiny \citep{Plassier2024:Efficient}}              & \textcolor{color4}{$0.09$}  & \textcolor{color6}{$9.71$}  & \textcolor{color4}{$7.49$}  & \textcolor{color4}{$2.31$}  & \textcolor{color6}{$2.03$}  & \textcolor{color6}{$20.72$} & \textcolor{color6}{$2.15$}  & \textcolor{color6}{$38.27$} & \textcolor{color2}{$8.55$}  \\
                  & FWCP-QQ          & \textcolor{color4}{$0.10$}  & \textcolor{color4}{$12.20$} & \textcolor{color4}{$8.35$}  & \textcolor{color4}{$2.36$}  & \textcolor{color4}{$2.31$}  & \textcolor{color4}{$22.76$} & \textcolor{color4}{$2.49$}  & \textcolor{color4}{$43.39$} & \textcolor{color2}{$8.61$}  \\
                  & PFWCP     & \textcolor{color4}{$0.10$}  & \textcolor{color4}{$10.58$} & \textcolor{color4}{$7.05$}  & \textcolor{color4}{$2.33$}  & \textcolor{color4}{$\textbf{2.16}$}  & \textcolor{color4}{$\textbf{21.37}$} & \textcolor{color4}{$\textbf{2.26}$}  & \textcolor{color4}{$\textbf{39.60}$} & \textcolor{color4}{$\textbf{9.05}$}  \\
                  & osPFWCP    & \textcolor{color6}{$0.09$}  & \textcolor{color4}{$9.80$}  & \textcolor{color4}{$\textbf{6.45}$}  & \textcolor{color4}{$2.29$}  & \textcolor{color2}{$1.92$}  & \textcolor{color6}{$19.66$} & \textcolor{color6}{$2.07$}  & \textcolor{color6}{$38.11$} & \textcolor{color2}{$8.86$}  \\
                  & osPFWGLCP & \textcolor{color4}{$\textbf{0.08}$}  & \textcolor{color4}{$\textbf{9.70}$}  & \textcolor{color6}{$6.22$}  & \textcolor{color4}{$\textbf{2.22}$}  & \textcolor{color6}{$2.06$}  & \textcolor{color6}{$18.88$} & \textcolor{color6}{$2.02$}  & \textcolor{color6}{$37.29$} & \textcolor{color2}{$8.53$}  \\
    \bottomrule
    \end{tabular}
    \vspace{-0.3cm}
\end{table}

Results across all the datasets are summarized in Table~\ref{tab:reg_marg_main}. We observe from the results that the proposed method \emph{PFWCP} satisfied the posed constraint of \gls{mc} above $0.9$ for all the datasets, while also yielding a generally low CMC, in fact the smallest in five out of the nine datasets, and being also generally efficient, in fact the most efficient in five out of the nine datasets. The proposed one-shot methods both with and without the localized scores also perform well on some of the datasets, although the decreased flexibility of the method results in a loss of valid coverage in a few cases. Integrating the localization of the scores into the proposed framework has potential to improve the efficiency as highlighted from the results on the synthetic datasets as well as the \emph{protein} dataset. None of the benchmarks are capable of reaching the low miscoverage and high efficiency of the proposed methods, and it is clearly seen that the low calibration data size in the cases of \emph{CP} and \emph{FGLCP}, due to only using local calibration data, has severe negative implications on this. Finally, the \emph{FCP-QQ} method often fails to provide the desired coverage level, which is due to the data heterogeneity. Additional complementary results with either lessened data heterogeneity is reported in Section~\ref{sec:additional_results}, where we have also included additional benchmarks that were omitted here for conciseness.

\paragraph{Calibration-conditional coverage:}
Results with all the benchmarks providing \gls{ccc} guarantees across the datasets are summarized in Table~\ref{tab:reg_cond_main}, when setting the significance levels to $\alpha=0.1$ and $\delta=0.1$. We observe that the proposed method \emph{PFWCP} on most of the datasets, exceptions are the \emph{crime} and \emph{cifar10c}, the specified \gls{ccc} constraint is realized, and that the gap in \gls{ccc} on the \emph{crime} and \emph{cifar10c} datasets is relatively small compared to gap of the other \gls{fcp} benchmarks. At the same time, the proposed method \emph{PFWCP} achieves a much lower conditional miscoverage and a much higher efficiency than the \emph{FGLCP} benchmark, as this is limited to calibrating only with local data. The \emph{osPFWCP} method has severe gaps in \gls{ccc} showcasing the limitations due to the lessened parameter flexibility, and a similar same observation is made for the \emph{FWCP} benchmark. On the other hand, the \emph{osPFWGLCP} method performs relatively well with just minor gaps in \gls{ccc} while maintaining small conditional miscoverage and a high efficiency.
\setlength{\tabcolsep}{0.3em} 
\begin{table}[t]
    \centering
    \caption{Calibration-conditional coverage (CCC), conditional miscoverage (CMC), and efficiency (Eff.). Methods with CCC above $1-\delta=0.9$, above $0.8$, and below $0.8$ are marked by colors \textcolor{color4}{green}, \textcolor{color6}{yellow}, and \textcolor{color2}{red}, respectively. The best performing method in each dataset, that satisfy the constraint of CCC above $0.9$, is highlighted with boldface.}
    \label{tab:reg_cond_main}
    \footnotesize
    \begin{tabular}{llrrrrrrrrr}
    \toprule
                  &                    & Gaussian                    & Poisson                     & airfoil                     & protein                     & bike                        & concrete                    & crime                       & star                        & cifar10c                     \\
    \midrule
    CCC (\%)   & CP                 & \textcolor{color6}{$86.60$}  & \textcolor{color6}{$87.20$} & \textcolor{color6}{$86.00$} & \textcolor{color6}{$86.80$} & \textcolor{color6}{$84.80$} & \textcolor{color6}{$86.40$} & \textcolor{color6}{$84.80$} & \textcolor{color6}{$86.80$} & \textcolor{color6}{$86.80$} \\
                  & FGLCP {\tiny \citep{Min2025:Personalized}}             & \textcolor{color4}{$\textbf{90.40}$}  & \textcolor{color4}{$93.00$} & \textcolor{color4}{$\textbf{91.20}$} & \textcolor{color4}{$\textbf{90.80}$} & \textcolor{color4}{$\textbf{90.40}$} & \textcolor{color4}{$\textbf{91.00}$} & \textcolor{color6}{$89.40$} & \textcolor{color4}{$\textbf{90.60}$} & \textcolor{color4}{$\textbf{91.80}$} \\
                  & FCP-QQ {\tiny \citep{Humbert2023:One}}          & \textcolor{color4}{$100.00$} & \textcolor{color6}{$85.60$} & \textcolor{color2}{$42.20$} & \textcolor{color6}{$84.00$} & \textcolor{color2}{$42.40$} & \textcolor{color2}{$7.60$}  & \textcolor{color6}{$80.40$} & \textcolor{color2}{$67.40$} & \textcolor{color2}{$62.40$} \\
                  & FWCP {\tiny \citep{Plassier2024:Efficient}}              & \textcolor{color2}{$79.40$}  & \textcolor{color6}{$89.40$}  & \textcolor{color4}{$97.80$}  & \textcolor{color6}{$89.20$} & \textcolor{color2}{$77.40$} & \textcolor{color2}{$74.60$} & \textcolor{color4}{$96.80$} & \textcolor{color6}{$83.00$} & \textcolor{color2}{$65.60$} \\
                  & FWCP-QQ          & \textcolor{color4}{$97.80$}  & \textcolor{color4}{$98.80$} & \textcolor{color4}{$99.60$} & \textcolor{color4}{$95.00$} & \textcolor{color4}{$95.80$} & \textcolor{color4}{$96.40$} & \textcolor{color4}{$\textbf{96.00}$} & \textcolor{color4}{$97.00$} & \textcolor{color2}{$61.40$} \\
                  & PFWCP     & \textcolor{color4}{$99.20$}  & \textcolor{color4}{$96.60$} & \textcolor{color4}{$97.00$} & \textcolor{color4}{$92.80$} & \textcolor{color4}{$91.60$} & \textcolor{color4}{$\textbf{91.00}$} & \textcolor{color6}{$85.60$} & \textcolor{color4}{$92.80$} & \textcolor{color6}{$88.40$} \\
                  & osPFWCP   & \textcolor{color4}{$93.20$}  & \textcolor{color4}{$\textbf{90.40}$} & \textcolor{color6}{$88.20$} & \textcolor{color6}{$87.80$} & \textcolor{color2}{$63.60$} & \textcolor{color2}{$77.20$} & \textcolor{color2}{$69.20$} & \textcolor{color6}{$81.40$} & \textcolor{color2}{$77.40$} \\
                  & osPFWGLCP & \textcolor{color4}{$95.00$}  & \textcolor{color4}{$95.60$} & \textcolor{color6}{$88.00$} & \textcolor{color4}{$91.80$} & \textcolor{color6}{$87.40$} & \textcolor{color6}{$88.40$} & \textcolor{color6}{$86.20$} & \textcolor{color4}{$94.80$} & \textcolor{color6}{$87.80$} \\\midrule
    CMC (\%) & CP                 & \textcolor{color6}{$3.37$}   & \textcolor{color6}{$3.61$}  & \textcolor{color6}{$3.28$}  & \textcolor{color6}{$3.66$}  & \textcolor{color6}{$3.39$}  & \textcolor{color6}{$3.40$}  & \textcolor{color6}{$3.42$}  & \textcolor{color6}{$3.52$}  & \textcolor{color6}{$3.42$}  \\
                  & FGLCP {\tiny \citep{Min2025:Personalized}}             & \textcolor{color4}{$5.02$}   & \textcolor{color4}{$5.16$}  & \textcolor{color4}{$4.85$}  & \textcolor{color4}{$5.11$}  & \textcolor{color4}{$4.70$}  & \textcolor{color4}{$5.04$}  & \textcolor{color6}{$4.92$}  & \textcolor{color4}{$4.94$}  & \textcolor{color4}{$\textbf{5.05}$}  \\
                  & FCP-QQ {\tiny \citep{Humbert2023:One}}          & \textcolor{color4}{$5.89$}   & \textcolor{color6}{$1.99$}  & \textcolor{color2}{$1.60$}  & \textcolor{color6}{$1.80$}  & \textcolor{color2}{$1.41$}  & \textcolor{color2}{$3.37$}  & \textcolor{color6}{$1.68$}  & \textcolor{color2}{$1.62$}  & \textcolor{color2}{$2.28$}  \\
                  & FWCP {\tiny \citep{Plassier2024:Efficient}}              & \textcolor{color2}{$2.31$}   & \textcolor{color6}{$2.85$}   & \textcolor{color4}{$4.34$}   & \textcolor{color6}{$1.82$}  & \textcolor{color2}{$2.47$}  & \textcolor{color2}{$1.78$}  & \textcolor{color4}{$6.26$}  & \textcolor{color6}{$2.14$}  & \textcolor{color2}{$4.49$}  \\
                  & FWCP-QQ          & \textcolor{color4}{$4.29$}   & \textcolor{color4}{$5.05$}  & \textcolor{color4}{$6.32$}  & \textcolor{color4}{$2.63$}  & \textcolor{color4}{$3.68$}  & \textcolor{color4}{$3.58$}  & \textcolor{color4}{$\textbf{5.97}$}  & \textcolor{color4}{$3.64$}  & \textcolor{color2}{$4.54$}  \\
                  & PFWCP     & \textcolor{color4}{$4.99$}   & \textcolor{color4}{$3.70$}  & \textcolor{color4}{$\textbf{3.74}$}  & \textcolor{color4}{$\textbf{2.42}$}  & \textcolor{color4}{$\textbf{2.83}$}  & \textcolor{color4}{$\textbf{2.79}$}  & \textcolor{color6}{$3.00$}  & \textcolor{color4}{$\textbf{2.65}$}  & \textcolor{color6}{$4.04$}  \\
                  & osPFWCP   & \textcolor{color4}{$\textbf{2.93}$}   & \textcolor{color4}{$\textbf{2.58}$}  & \textcolor{color6}{$2.41$}  & \textcolor{color6}{$2.05$}  & \textcolor{color2}{$1.74$}  & \textcolor{color2}{$1.94$}  & \textcolor{color2}{$1.94$}  & \textcolor{color6}{$2.07$}  & \textcolor{color2}{$2.60$}  \\
                  & osPFWGLCP  & \textcolor{color4}{$4.09$}   & \textcolor{color4}{$3.46$}  & \textcolor{color6}{$2.63$}  & \textcolor{color4}{$2.57$}  & \textcolor{color6}{$2.43$}  & \textcolor{color6}{$2.97$}  & \textcolor{color6}{$2.57$}  & \textcolor{color4}{$2.99$}  & \textcolor{color6}{$5.06$}  \\\midrule
    Eff.         & CP                 & \textcolor{color6}{$0.05$}   & \textcolor{color6}{$10.75$} & \textcolor{color6}{$7.21$}  & \textcolor{color6}{$2.64$}  & \textcolor{color6}{$2.38$}  & \textcolor{color6}{$25.99$} & \textcolor{color6}{$2.53$}  & \textcolor{color6}{$43.90$} & \textcolor{color6}{$9.28$}  \\
                  & FGLCP {\tiny \citep{Min2025:Personalized}}             & \textcolor{color4}{$0.06$}   & \textcolor{color4}{$22.93$} & \textcolor{color4}{$9.00$}  & \textcolor{color4}{$2.97$}  & \textcolor{color4}{$3.50$}  & \textcolor{color4}{$35.70$} & \textcolor{color6}{$3.07$}  & \textcolor{color4}{$47.68$} & \textcolor{color4}{$\textbf{9.46}$}  \\
                  & FedCP-QQ {\tiny \citep{Humbert2023:One}}          & \textcolor{color4}{$0.05$}   & \textcolor{color6}{$9.72$}  & \textcolor{color2}{$5.94$}  & \textcolor{color6}{$2.45$}  & \textcolor{color2}{$1.96$}  & \textcolor{color2}{$16.81$} & \textcolor{color6}{$2.22$}  & \textcolor{color2}{$40.15$} & \textcolor{color2}{$9.05$}  \\
                  & FWCP {\tiny \citep{Plassier2024:Efficient}}              & \textcolor{color2}{$0.04$}   & \textcolor{color6}{$10.24$}  & \textcolor{color4}{$7.86$}   & \textcolor{color6}{$2.46$}  & \textcolor{color2}{$2.21$}  & \textcolor{color2}{$21.92$} & \textcolor{color4}{$\textbf{3.77}$}  & \textcolor{color6}{$41.72$} & \textcolor{color2}{$9.16$}  \\
                  & FWCP-QQ          & \textcolor{color4}{$0.05$}   & \textcolor{color4}{$14.61$} & \textcolor{color4}{$9.39$}  & \textcolor{color4}{$2.55$}  & \textcolor{color4}{$2.64$}  & \textcolor{color4}{$26.08$} & \textcolor{color4}{$3.78$}  & \textcolor{color4}{$46.32$} & \textcolor{color2}{$9.11$}  \\
                  & PFWCP     & \textcolor{color4}{$0.05$}   & \textcolor{color4}{$11.16$} & \textcolor{color4}{$\textbf{7.33}$}  & \textcolor{color4}{$2.53$}  & \textcolor{color4}{$\textbf{2.39}$}  & \textcolor{color4}{$\textbf{24.34}$} & \textcolor{color6}{$2.66$}  & \textcolor{color4}{$43.05$} & \textcolor{color6}{$9.36$}  \\
                  & osPFWCP   & \textcolor{color4}{$\textbf{0.05}$}   & \textcolor{color4}{$\textbf{10.06}$} & \textcolor{color6}{$6.70$}  & \textcolor{color6}{$2.48$}  & \textcolor{color2}{$2.08$}  & \textcolor{color2}{$22.11$} & \textcolor{color2}{$2.28$}  & \textcolor{color6}{$41.40$} & \textcolor{color2}{$9.17$}  \\
                  & osPFWGLCP & \textcolor{color4}{$0.06$}   & \textcolor{color4}{$11.34$} & \textcolor{color6}{$7.12$}  & \textcolor{color4}{$\textbf{2.50}$}  & \textcolor{color6}{$2.42$}  & \textcolor{color6}{$25.06$} & \textcolor{color6}{$2.39$}  & \textcolor{color4}{$\textbf{41.80}$} & \textcolor{color6}{$9.44$}  \\
    \bottomrule
    \end{tabular}
    \vspace{-0.3cm}
\end{table}

\section{Conclusion}\label{sec:conclusion}
Personalized federated weighted conformal prediction closes an important gap between privacy-preserving collaboration of heterogeneous agents and reliable uncertainty quantification. By accounting for covariate shifts in both coverage bias and variance adjustments, it produces efficient prediction sets with more stable coverage compared to existing federated conformal methods. The presented framework motivates further studies to improve robustness of federated conformal prediction.



\newpage
\appendix

\renewcommand{\thesection}{S\arabic{section}}

\title{Supplementary Material for\\"Multi-Agent Conformal Prediction with Personalized Statistical Validity"}

\author{%
  Martin V. Vejling \\
  Department of Mathematical Sciences and \\
  Department of Electronic Systems\\
  Aalborg University\\
  Aalborg, Denmark \\
  \texttt{mvv@\{math,es\}.aau.dk} \\
  \And
  Shashi Raj Pandey \\
  Department of Electronic Systems \\
  Aalborg University \\
  Aalborg, Denmark \\
  \texttt{srp@es.aau.dk} \\
  \AND
  Christophe A. N. Biscio \\
  Department of Mathematical Sciences \\
  Aalborg University \\
  Aalborg, Denmark \\
  \texttt{christophe@math.aau.dk} \\
  \And
  Petar Popovski \\
  Department of Electronic Systems \\
  Aalborg University \\
  Aalborg, Denmark \\
  \texttt{petarp@es.aau.dk} \\}

\maketitle

\section{Proofs}\label{sec:proofs}
\begin{proof}[Proof of Proposition~\ref{prop:coverage_approximation}]
    We prove that
    \begin{equation*}
        \mathbb{P}\big(Y\in\mathcal{C}_{\beta,\tau}^{\rm wqq}(X)|\mathcal{D}\big) = \mathbb{P}\big(\hat{s}(X,Y)\leq Q_{\beta,\tau}|\mathcal{D}\big)\stackrel{d}{\to} {\rm Quantile}\Big(1-\tau; \sum_{k=1}^K w_k \delta_{U_\beta^k}\Big).
    \end{equation*}
    To that end, consider the following
    \begin{equation*}
        \mathbb{P}(\hat{s}(X,Y)\leq Q_{\beta,\tau} | \mathcal{D}) = \mathbb{P}(F_{\hat{s}}^{-1}(U)\leq Q_{\beta,\tau} | \mathcal{D}) = F_{\hat{s}}(Q_{\beta,\tau}),
    \end{equation*}
    where $U\sim{\rm Uniform}([0,1])$, $F_{\hat{s}}$ is the \gls{cdf} of the scores and the second equality relies on the assumption of $F_{\hat{s}}$ being continuous. What remains to show is that $F_{\hat{s}}(Q_{\beta,\tau})$ converges in distribution to the term ${\rm Quantile}\big(1-\tau; \sum_{k=1}^K w_k \delta_{U_\beta^k}\big)$.
    Observe that there is an equivalence between approximating the empirical \gls{cdf} and the empirical quantile. Hence, by Lemma~\ref{lem:empirical_cdf_approximation}, $F_{\hat{s}}(Q_\beta^k)$ is asymptotically distributed as $U_\beta^k \sim {\rm Beta}\big((1-\beta)(n_{\rm eff}^k+1),\beta(n_{\rm eff}^k+1)\big)$, noting that this is the distribution of the $1-\beta$ quantile of $n_{\rm eff}^k$ independent standard uniform random variables. Using that $F_{\hat{s}}^{-1}$ is nondecreasing, this implies that \citep{Humbert2024:Marginal}
    \begin{align*}
        Q_{\beta,\tau} &= {\rm Quantile}\Big(1-\tau;\sum_{k=1}^K w_k \delta_{Q_\beta^k}\Big) \stackrel{d}{\to} {\rm Quantile}\Big(1-\tau;\sum_{k=1}^K w_k \delta_{F_{\hat{s}}^{-1}(U_\beta^k)}\Big)\\
        &= F_{\hat{s}}^{-1}\Big({\rm Quantile}\Big(1-\tau;\sum_{k=1}^K w_k \delta_{U_\beta^k}\Big)\Big),
    \end{align*}
    concluding the proof.
    
\end{proof}
We will now introduce some lemmas which will be required in the proof of Lemma~\ref{lem:empirical_cdf_approximation}. First, we state a known, elementary upper bound on Jensen's gap.
\begin{lemma}\label{lem:Jensen_gap}
    Suppose that $h : \mathbb{R} \to \mathbb{R}$ is a twice differentiable convex function for which $h'' < \Lambda$ for some constant $\Lambda \in \mathbb{R}$, and let $X$ be a real-valued random variable. Then, Jensen's gap is upper bounded by
    \begin{equation*}
        \mathbb{E}[h(X)] - h(\mathbb{E}[X]) \leq \frac{1}{2}\Lambda {\rm Var}[X].
    \end{equation*}
\end{lemma}

Next we introduce Lemma~\ref{lem:bounded_variance_N} and Lemma~\ref{lem:Lambda_order} which are used in Lemma~\ref{lem:empirical_cdf_approximation} to show that Jensen's gap, in this case, vanishes asymptotically.

\begin{lemma}\label{lem:bounded_variance_N}
    Let $X_j \stackrel{iid}{\sim} \mathbb{P}_X^{k}$ for $j\in[n]$ and some $k\in\{2,\dots,K\}$, $X \sim \mathbb{P}_X$, and $\omega(X) = [d\mathbb{P}_X/d\mathbb{P}_X^{k}](X)$. Assume that there exists a bounded constant $B>0$ such that $\mathbb{P}_X(\omega(X) \leq B) = 1$ and $\mathbb{P}_X^k(\omega(X_j) \leq B) = 1$ for $j\in[n]$.
    Then, ${\rm Var}\big[\sum_{j=1}^{n} \omega(X_j) + \omega(X)\big] = O(n)$.
\end{lemma}
\begin{proof}
    Consider the following which holds by independence of $X_j$:
    \begin{equation*}
        {\rm Var}\Big[\sum_{j=1}^{n} \omega(X_j) + \omega(X)\Big] = \sum_{j=1}^{n} {\rm Var}[\omega(X_j)] + {\rm Var}[\omega(X)]\\
    \end{equation*}
    By identical distribution of $X_j$,
    \begin{equation*}
        {\rm Var}\Big[\sum_{j=1}^{n} \omega(X_j) + \omega(X)\Big] = n\big(\mathbb{E}[\omega^2(X_1)] - \mathbb{E}[\omega(X_1)]^2\big) + \mathbb{E}[\omega^2(X)] - \mathbb{E}[\omega(X)]^2.
    \end{equation*}
    Finally, almost sure boundedness of the weights implies that $\mathbb{E}[\omega(X_1)] \leq B$ and $\mathbb{E}[\omega^2(X_1)] \leq B^2$, which concludes the proof.
\end{proof}

\begin{lemma}\label{lem:Lambda_order}
    Let $X_j \sim \mathbb{P}_X^{k}$ for $j=1,\dots,n$ and some $k\in\{2,\dots,K\}$, $X \sim \mathbb{P}_X$, and $\omega(X) = [d\mathbb{P}_X/d\mathbb{P}_X^{k}](X)$. Assume that there exist bounded constants $0<A\leq1\leq B$ such that $\mathbb{P}_X(\omega(X) \in [A, B]) = 1$ and $\mathbb{P}_X^{k}(\omega(X_j) \in [A, B]) = 1$.
    Then, the smallest constant $\Lambda_n$, such that for any given $n$,
    \begin{equation*}
        \mathbb{P}\bigg(\Big(\sum_{j=1}^{n} \omega(X_j) + \omega(X)\Big)^{-d} < \Lambda_n\bigg) = 1
    \end{equation*}
    is of order $O(1/n^d)$ for $d\geq 1$.
\end{lemma}
\begin{proof}
    As $\omega(X) \in [A,B]$ almost surely
    \begin{equation*}
        \mathbb{P}\Big(\sum_{j=1}^{n} \omega(X_j) + \omega(X) \geq A(n+1)\Big) = 1,
    \end{equation*}
    and so
    \begin{equation*}
        \mathbb{P}\bigg(\Big(\sum_{j=1}^{n} \omega(X_j) + \omega(X)\Big)^{-d} \leq A^{-d}(n+1)^{-d}\bigg) = 1
    \end{equation*}
    Hence, $\Lambda_n$ is necessarily upper bounded by $A^{-d}(n+1)^{-d}$. Similarly, $\Lambda_n$ is lower bounded by $B^{-d}(n+1)^{-d}$, concluding the proof.
\end{proof}

The final lemmas before proceeding to the proof of Lemma~\ref{lem:empirical_cdf_approximation} provides statements related to the convergence of the normalization factor $\sum_{j=1}^{n} \omega(X_j) + \omega(X)$ and the effective sample size $\Vert [\omega(X_1),\dots,\omega(X_n)]\Vert_1^2/\Vert [\omega(X_1),\dots,\omega(X_n)]\Vert_2^2$ when replacing a calibration datum with a test datum. These results will be necessary to apply McDiarmid's tail bound in the proof of Lemma~\ref{lem:empirical_cdf_approximation}.
\begin{lemma}\label{lem:convergence_N}
    Let $X_j \sim \mathbb{P}_X^{k}$ for $j=1,\dots,n$ and some $k\in\{2,\dots,K\}$, $\Tilde{X}, X \sim \mathbb{P}_X$, and $\omega(X) = [d\mathbb{P}_X/d\mathbb{P}_X^{k}](X)$. Assume that there exists bounded constants $0<A\leq1\leq B$ such that $\mathbb{P}_X(\omega(X) \in[A,B]) = 1$ and $\mathbb{P}_X^k(\omega(X_j) \in[A,B]) = 1$.
    Then, as $n \to \infty$
    \begin{equation*}
        \mathbb{P}\bigg(\bigg|\frac{1}{\sum_{j=1}^{n} \omega(X_j) + \omega(X)} - \frac{1}{\sum_{j=2}^{n} \omega(X_j) + \omega(\Tilde{X}) + \omega(X)}\bigg| > \epsilon\bigg) \to 0,
    \end{equation*}
    for any $\epsilon > 0$.
\end{lemma}
\begin{proof}
    First, notice that almost surely
    \begin{align*}
        &\bigg|\frac{1}{\sum_{j=1}^{n} \omega(X_j) + \omega(X)} - \frac{1}{\sum_{j=2}^{n} \omega(X_j) + \omega(\Tilde{X}) + \omega(X)}\bigg|\\
        &\qquad\qquad= \bigg|\frac{\omega(\Tilde{X}) - \omega(X_1)}{\Big(\sum_{j=1}^n \omega(X_j) + \omega(X)\Big)\Big(\sum_{j=2}^n \omega(X_j) + \omega(\Tilde{X}) + \omega(X)\Big)}\bigg|\\
        &\qquad\qquad\leq \frac{B-A}{A^2(n+1)^2}.
    \end{align*}
    It follows that
    \begin{equation*}
        \mathbb{P}\bigg(\bigg|\frac{1}{\sum_{j=1}^{n} \omega(X_j) + \omega(X)} - \frac{1}{\sum_{j=2}^{n} \omega(X_j) + \omega(\Tilde{X}) + \omega(X)}\bigg|>\epsilon\bigg) \leq \mathbb{P}\Big(\frac{B-A}{A^2(n+1)^2}>\epsilon\Big).
    \end{equation*}
    The probability of the trivial event $\frac{B-A}{A^2(n+1)^2}>\epsilon$ goes to zero when $\frac{B-A}{A^2(n+1)^2}\to0$ which occurs as $n\to \infty$.
\end{proof}

\begin{lemma}\label{lem:convergence_M}
    Let $X_j \sim \mathbb{P}_X^{k}$ for $j=1,\dots,n$ and some $k\in\{2,\dots,K\}$, $X \sim \mathbb{P}_X$, and $\omega(X) = [d\mathbb{P}_X/d\mathbb{P}_X^{k}](X)$. Assume that there exists bounded constants $0<A\leq1\leq B$ such that $\mathbb{P}_X(\omega(X) \in[A,B]) = 1$ and $\mathbb{P}_X^k(\omega(X_j) \in [A,B]) = 1$.
    Then, as $n \to \infty$
    \begin{equation*}
        \mathbb{P}\Bigg(\Bigg|\frac{\Vert [\omega(X_1),\dots,\omega(X_n)]\Vert_2^2}{\Vert [\omega(X_1),\dots,\omega(X_n)]\Vert_1^2} - \frac{\Vert [\omega(X),\omega(X_2),\dots,\omega(X_n)]\Vert_2^2}{\Vert [\omega(X), \omega(X_2),\dots,\omega(X_n)]\Vert_1^2}
        \Bigg| > \epsilon\Bigg) \to 0,
    \end{equation*}
    for any $\epsilon > 0$.
\end{lemma}
\begin{proof}
    The result follows from similar arguments as for Lemma~\ref{lem:convergence_N}.
\end{proof}

We are now ready to present the core technical lemma used in the proof of Proposition~\ref{prop:coverage_approximation}.
\begin{lemma}\label{lem:empirical_cdf_approximation}
    Let $(X_i,Y_i)\stackrel{iid}{\sim}\mathbb{P}^{k}=\mathbb{P}_{Y|X}\mathbb{P}_X^{k}$ for $i\in[n]$ and some $k\in\{2,\dots,K\}$ and $(X,Y),(X_i^{\rm te}, Y_i^{\rm te}) \sim \mathbb{P}=\mathbb{P}_{Y|X}\mathbb{P}_X$. Let $\hat{s}_i^{\rm te} = \hat{s}(X_i^{\rm te}, Y_i^{\rm te})$ and $\hat{s}_i = \hat{s}(X_i, Y_i)$ be pre-trained scores. Assuming that $\mathbb{P}_X$ is absolutely continuous with respect to $\mathbb{P}_X^{k}$, let $\omega(X) = [d\mathbb{P}_X/d\mathbb{P}_X^{k}](X)$ be the density ratio, and assume further that there exists bounded constants $0<A\leq 1 \leq B$ such that $\mathbb{P}_X(\omega(X) \in [A, B]) = 1$ and $\mathbb{P}_X^k(\omega(X_i) \in [A, B]) = 1$.
    Moreover, denote by $n_{\rm eff} = \lceil \Vert [\omega(X_1),\dots,\omega(X_n)]\Vert_1^2/\Vert [\omega(X_1),\dots,\omega(X_n)]\Vert_2^2\rceil$ the effective sample size. Then, as $n \to \infty$,
    \begin{equation*}
        \sup_{W \in \sigma(\mathbb{R})} \bigg|\frac{\sum_{i=1}^n \omega(X_i) \delta_{\hat{s}_i}(W) + \omega(X)\delta_{\infty}(W)}{\sum_{i=1}^n \omega(X_i) + \omega(X)} - \frac{\sum_{i=1}^{n_{\rm eff}} \delta_{\hat{s}_i^{\rm te}}(W) + \delta_{\infty}(W)}{n_{\rm eff}+1}\bigg| \to 0,
    \end{equation*}
    where $\sigma(\mathbb{R})$ denotes the set of left-infinite half intervals $W=(-\infty,x]$.
\end{lemma}
\begin{proof}
    The proof is inspired by the proof of Lemma 8.5 of \cite{Gretton2009:Dataset}. Before proceeding with the proof, consider the following shorthand notations: $Z_j = (X_j,Y_j)$, $Z_j^{\rm te} = (X_j^{\rm te}, Y_j^{\rm te})$, $\Tilde{Z} = (\Tilde{X},\Tilde{Y}$), $\breve{Z} = (\breve{X},\breve{Y})$, $\omega = \omega(X)$, $\omega_j = \omega(X_j), \Tilde{\omega} = \omega(\Tilde{X})$, $\breve{\omega} = \omega(\breve{X})$, $\omega_{1:n} = [\omega_1,\dots,\omega_n]$, $\delta_j = \delta_{\hat{s}(Z_j)}$, $\Tilde{\delta} = \delta_{\hat{s}(\Tilde{Z})}$, $\breve{\delta} = \delta_{\hat{s}(\breve{Z})}$, $N = N(X_1,\dots,X_n, X) = \sum_{j=1}^{n} \omega_j + \omega$, $M = M(X_1,\dots,X_n) = n_{\rm eff} + 1$, and $NM(X_1,\dots,X_n, X) = N(X_1,\dots,X_n, X)M(X_1,\dots,X_n)$. Defining also the following shorthand expressions: $\Delta = \Delta(Z_1,\dots,Z_n, X) = \frac{1}{N} \Big(\sum_{i=1}^n \omega_i \delta_i(W) + \omega\delta_{\infty}(W)\Big)$ for any $W \in \sigma(\mathbb{R})$, $\Delta^{\rm te} = \Delta^{\rm te}(Z_1, \dots, Z_n, Z_1^{\rm te}, \dots, Z_{M-1}^{\rm te}) = \frac{1}{M} \Big(\sum_{i=1}^{M-1} \delta_i^{\rm te}(W) + \delta_{\infty}(W)\Big)$, and $\Xi = \Xi(Z_1, \dots, Z_n, X, Z_1^{\rm te}, \dots, Z_{M-1}^{\rm te}) = \big|\Delta - \Delta^{\rm te}\big|$.
    We will apply McDiarmid's tail bound \citep{McDiarmid1989:Differences} to show that
    \begin{equation*}
        \mathbb{P}\big(\sup_W|\Xi - \mathbb{E}[\Xi]| > \epsilon\big) < 2{\rm exp}\bigg(\frac{-2\epsilon^2}{\sum_{i=1}^{n+M} c_i^2}\bigg),
    \end{equation*}
    for any $\epsilon > 0$ and where $c_i$ are such that $\sup_W|\Xi - \Xi_{-i}| \leq c_i$ where $\Xi_{-i} = \Xi(Z_1, \dots, Z_{i-1}, \Tilde{Z}, Z_{i+1},\dots, Z_n, X, Z_1^{\rm te}, \dots, Z_{M-1}^{\rm te})$ for $i\in[n+M]$. For this we need to bound the change in $\Xi$ if we replace any $(X_i, Y_i)$ by an arbitrary $\Tilde{Z}=(\Tilde{X}, \Tilde{Y})$ and likewise if we replace any $(X_i^{\rm te}, Y_i^{\rm te})$ by some $\Tilde{Z}=(\Tilde{X}, \Tilde{Y})$. For $i\in[n]$, by the triangle inequality,
    \begin{equation*}
        \sup_W \big|\Xi - \Xi_{-i}\big| = \sup_W \big|\Delta - \Delta^{\rm te} - (\Delta_{-i} - \Delta^{\rm te}_{-i})\big| \leq \sup_W \big|\Delta - \Delta_{-i}\big| + \big|\Delta^{\rm te} - \Delta_{-i}^{\rm te}\big|,
    \end{equation*}
    where $\Delta_{-i} = \Delta(Z_1, \dots, Z_{i-1}, \Tilde{Z}, Z_{i+1}, \dots, Z_n, X)$ for $i\in[n+1]$, and $\Delta^{\rm te}_{-i} = \Delta^{\rm te}(Z_1, \dots, Z_n, Z_1^{\rm te}, \dots, Z_{i-n-2}^{\rm te}, \Tilde{Z}, Z_{i-n}^{\rm te}, \dots, Z_{M-1}^{\rm te})$ for $i\in [n+M]\setminus\{n+1\}$. For $i\in[n+1]$, define $\Delta_N^i = \Tilde{\omega} - \omega_i$, and consider the first term
    \begin{equation*}
        \begin{split}
            \sup_W \big|\Delta - \Delta_{-i}\big|
            &= \sup_W\Big|\frac{1}{N} \Big(\sum_{j=1}^n \omega_j \delta_j(W) + \omega\delta_{\infty}(W)\Big) \\
            &\qquad - \frac{1}{N+\Delta_N^i} \Big(\sum_{j=1,j\neq i}^n \omega_j \delta_j(W) + \Tilde{\omega} \Tilde{\delta}(W)+ \omega\delta_{\infty}(W)\Big)\Big|.
        \end{split}
    \end{equation*}
    By the boundedness of the weights and using Lemma~\ref{lem:convergence_N}, for $\epsilon'>0$ there exists $n'$ such that for $n\geq n'$ it holds that $|1/N - 1/(N+\Delta_N^i)| < \epsilon'$. With this, for $n$ sufficient large
    \begin{equation*}
        \sup_W \big|\Delta - \Delta_{-i}\big| = \frac{1}{N} \sup_W\Big|\omega_i \delta_i(W) - \omega(\Tilde{X}) \Tilde{\delta}(W)\Big| + O(\epsilon') = O(1/n) + O(\epsilon').
    \end{equation*}
    Define now $M_{-i} = M(X_1,\dots,X_{-i},\Tilde{X},X_i+1,\dots,X_n)$ for $i\in[n]$, let
    \begin{equation*}
        \Delta_M^i(W) = \mathbbm{1}[M_{-i}\geq M] \sum_{j=M}^{M_{-i}-1} \delta_j^{\rm te}(W) + \mathbbm{1}[M_{-i}< M] \sum_{j=M_{-i}}^{M-1} \delta_j^{\rm te}(W),
    \end{equation*}
    denote by $m(W) = \sum_{j=1}^{M-1} \delta_j^{\rm te}(W) + \delta_{\infty}(W)$, and consider the second term
    \begin{equation*}
        \sup_W \big|\Delta^{\rm te} - \Delta_{-i}^{\rm te}\big| = \sup_W \Big|\frac{m(W)}{M} - \frac{m(W)+\Delta_M^i(W)}{M+\Delta_M^{i}(\mathbb{R})}\Big| = \sup_W \Big|\frac{m(W)\Delta_M^i(W)}{M(M+\Delta_M^i(\mathbb{R}))} - \frac{\Delta_M^i(W)}{M+\Delta_M^i(\mathbb{R})}\Big|.
    \end{equation*}
    By the boundedness of the weights and using Lemma~\ref{lem:convergence_M}, for any $\epsilon'' > 0$ there exists $n''$ such that for $n\geq n''$ it holds that $|1/M - 1/(M+\Delta_M^i(\mathbb{R}))| < \epsilon''$. Hence, for $n$ sufficiently large
    \begin{equation*}
        \sup_W \big|\Delta^{\rm te} - \Delta_{-i}^{\rm te}\big| = \sup_W \Big|\frac{m(W)\Delta_M^i(W)}{M(M+\Delta_M^i(\mathbb{R}))} - \frac{\Delta_M^i(W)}{M+\Delta_M^i(\mathbb{R})}\Big| + O(\epsilon'') = O(1/n) + O(\epsilon'').
    \end{equation*}
    Combining the terms we can conclude that for $i\in[n]$, for $n$ sufficient large $\sup_W \big|\Xi - \Xi_{-i}\big| = O(1/n) + O(\epsilon') + O(\epsilon'')$. Similarly, for $i=n+1$, $\sup_W \big|\Xi - \Xi_{-i}\big| = O(1/n) + O(\epsilon')$.
    
    For $i\in\{n+2,\dots,n+M\}$, we have
    \begin{align*}
        \sup_W \big|\Xi - \Xi_{-i}\big|
        &= \sup_W \big|\Delta - \Delta^{\rm te} - (\Delta - \Delta^{\rm te}_{-i})\big| = \sup_W \big|\Delta^{\rm te} - \Delta_{-i}^{\rm te}\big| \\
        &= \frac{1}{M} \sup_W\Big|\delta_{i-n-1}^{\rm te}(W) - \Tilde{\delta}(W)\Big| = O(1/n).
    \end{align*}
    Applying now McDiarmid's tail bound, for $n$ sufficiently large, it holds that
    \begin{equation*}
        \mathbb{P}\big(\sup_W|\Xi - \mathbb{E}[\Xi]| > \epsilon\big) < 2{\rm exp}\big(-2\epsilon^2n/c_{\epsilon',\epsilon''}\big),
    \end{equation*}
    for some constant $c_{\epsilon',\epsilon''} > 0$.
    Hence, with probability $1-\delta$ the deviation of the random variable from its expectation is bounded by
    \begin{equation*}
        \sup_W |\Xi - \mathbb{E}[\Xi]| \leq \sqrt{\frac{-c_{\epsilon',\epsilon''}{\rm ln}(\delta/2)}{2n}}.
    \end{equation*}
    We proceed to bound the expected value $\mathbb{E}[\Xi]$ using Jensen's inequality as $\mathbb{E}[\Xi] \leq \sqrt{\mathbb{E}[\Xi^2]}$. Expanding the expectation
    \begin{align}
        \sup_W\mathbb{E}[\Xi^2] &= \sup_W\mathbb{E}\Big[\Big(\frac{1}{N}\Big(\sum_{j=1}^n \omega_j\delta_{\hat{s}_j} + \omega\delta_\infty\Big) - \frac{1}{M} \Big(\sum_{j=1}^{M-1} \delta_{\hat{s}_j^{\rm te}} + \delta_\infty \Big)\Big)^2\Big] \nonumber\\
        \begin{split}
            &= \sup_W\mathbb{E}\Big[\frac{1}{N^2}\Big(\sum_{i,j=1}^n\omega_i\omega_j \delta_i\delta_j + 2\sum_{j=1}^n\omega_j\omega \delta_j \delta_\infty + \omega^2 \delta_\infty \delta_\infty\Big)\Big]\\
            & +\mathbb{E}\Big[\frac{1}{M^2} \Big(\sum_{i,j=1}^{M-1}  \delta_i^{\rm te} \delta_j^{\rm te} + 2\sum_{j=1}^{M-1} \delta_j^{\rm te} \delta_\infty +  \delta_\infty \delta_\infty\Big)\Big]\\
            & -2\mathbb{E}\Big[\frac{1}{NM}\Big(\sum_{i=1}^n\sum_{j=1}^{M-1} \omega_i \delta_i \delta_j^{\rm te} + \sum_{i=1}^n \omega_i \delta_i\delta_\infty + \sum_{j=1}^{M-1}\omega  \delta_j^{\rm te} \delta_\infty + \omega  \delta_\infty \delta_\infty\Big)\Big].
        \end{split}\label{eq:expanded_square}
    \end{align}
    Beginning with the first term, using the law of total expectation, and interchanging measures exploiting that $\mathbb{P}_X(X) = \omega(X) \mathbb{P}_X^k(X)$:
    \begin{align*}
        \begin{split}
            \sup_W\mathbb{E}\Big[\frac{1}{N^2}\sum_{i,j=1}^n\omega_i\omega_j\delta_i\delta_j\Big]
            &= \sup_Wn(n-1)\mathbb{E}\Big[\mathbb{E}\Big[\frac{ \Tilde{\delta}^{\rm te} \breve{\delta}^{\rm te}}{N^2(\Tilde{X}^{\rm te},\breve{X}^{\rm te},X_3,\dots,X_n,X)} | X_3,\dots,X_n,X\Big]\Big]\\
            &\quad + n\mathbb{E}\Big[\mathbb{E}\Big[\frac{\Tilde{\omega}^{\rm te} \Tilde{\delta}^{\rm te} \Tilde{\delta}^{\rm te}}{N^2(\Tilde{X}^{\rm te},X_2,\dots,X_n,X)} | X_2,\dots,X_n,X\Big]\Big]
        \end{split}\\
        \begin{split}
            &= \sup_W(n^2-n)\mathbb{E}\Big[\frac{ \Tilde{\delta}^{\rm te} \breve{\delta}^{\rm te}}{N^2(\Tilde{X}^{\rm te},\breve{X}^{\rm te},X_3,\dots,X_n,X)}\Big]\\
            &\quad + n \mathbb{E}\Big[\frac{\Tilde{\omega}^{\rm te} \Tilde{\delta}^{\rm te} \Tilde{\delta}^{\rm te}}{N^2(\Tilde{X}^{\rm te},X_2,\dots,X_n,X)}\Big]
        \end{split}\\
        \begin{split}
            &\leq (n^2-n)\mathbb{E}\Big[\frac{1}{N^2(\Tilde{X}^{\rm te},\breve{X}^{\rm te},X_3,\dots,X_n,X)}\Big]\\
            &\quad + n \mathbb{E}\Big[\frac{\Tilde{\omega}^{\rm te}}{N^2(\Tilde{X}^{\rm te},X_2,\dots,X_n,X)}\Big].
        \end{split}
    \end{align*}
    Moving on to the first term in the second line of Eq.~\eqref{eq:expanded_square}:
    \begin{align}
        \sup_W\mathbb{E}\Big[\frac{1}{M^2} \sum_{i,j=1}^{M-1}  \delta_i^{\rm te} \delta_j^{\rm te}\Big] &= \sup_W\mathbb{E}\Big[\frac{1}{M^2} \Big(\sum_{j=1}^{M-1} \mathbb{E}[ \delta_j^{\rm te} \delta_j^{\rm te} | X_1,\dots,X_n] + \sum_{\substack{i,j=1\\j\neq i}}^{M-1} \mathbb{E}[ \delta_i^{\rm te} \delta_j^{\rm te} | X_1,\dots,X_n]\Big)\Big]\nonumber\\
        &= \sup_W\mathbb{E}\Big[\frac{1}{M} - \frac{1}{M^2}\Big] \mathbb{E}[ \delta_1^{\rm te} \delta_1^{\rm te}] + \mathbb{E}\Big[1-\frac{3}{M}+\frac{2}{M^2}\Big] \mathbb{E}[ \delta_1^{\rm te} \delta_2^{\rm te}]\nonumber\\
        &\leq \mathbb{E}\Big[\frac{1}{M^2}\Big] - 2\mathbb{E}\Big[\frac{1}{M}\Big] + 1.\label{eq:remainder_one}
    \end{align}
    Now the first term in the third line of Eq.~\eqref{eq:expanded_square}:
    \begin{align*}
        \sup_W2\mathbb{E}\Big[\frac{1}{NM}\sum_{i=1}^n\sum_{j=1}^{M-1} \omega_i \delta_i, \delta_j^{\rm te}\Big]
        &= \sup_W2n\mathbb{E}\Big[\frac{\sum_{j=1}^{M(\Tilde{X}^{\rm te}, X_2,\dots,X_n)-1}  \Tilde{\delta}^{\rm te} \delta_j^{\rm te}}{NM(\Tilde{X}^{\rm te}, X_2,\dots,X_n, X)}\Big]\\
        &= \sup_W2n\mathbb{E}\Big[\frac{(M(\Tilde{X}^{\rm te}, X_2,\dots,X_n)-1) \Tilde{\delta}^{\rm te} \delta_1^{\rm te}}{NM(\Tilde{X}^{\rm te}, X_2,\dots,X_n, X)}\Big]\\
        \begin{split}
            &\leq 2n \Big(\mathbb{E}\Big[\frac{1}{N(\Tilde{X}^{\rm te}, X_2,\dots,X_n, X)}\Big]\\
            &\qquad + \mathbb{E}\Big[\frac{1}{NM(\Tilde{X}^{\rm te}, X_2,\dots,X_n, X)}\Big]\Big).
        \end{split}
    \end{align*}
    What remains is to argue that the terms $\mathbb{E}\big[\frac{n^2}{N^2(\Tilde{X}^{\rm te},\breve{X}^{\rm te},X_3,\dots,X_n,X)}\big]$ and $\mathbb{E}\big[\frac{-2n}{N(\Tilde{X}^{\rm te}, X_2,\dots,X_n, X)}\big]$ cancels with the $1$ of Eq.~\eqref{eq:remainder_one}.
    With Jensen's inequality we obtain that
    \begin{align*}
        \mathbb{E}\Big[\frac{n}{N(\Tilde{X}^{\rm te}, X_2,\dots,X_n, X)}\Big] &\geq \frac{n}{\mathbb{E}\big[N(\Tilde{X}^{\rm te}, X_2,\dots,X_n, X)\big]} = \frac{n}{2\mathbb{E}[\Tilde{\omega}^{\rm te}] + (n-1)\mathbb{E}[\omega_2]}\\
        &= \frac{n}{2\mathbb{E}[\Tilde{\omega}^{\rm te}] + (n-1)}.
    \end{align*}
    By regularity of the weight function, $\mathbb{E}[\Tilde{\omega}^{\rm te}]$ is bounded and hence $n/(2\mathbb{E}[\Tilde{\omega}^{\rm te}] + (n-1))\to 1$ as $n \to \infty$. Jensen's gap can be bounded by second order properties, see Lemma~\ref{lem:Jensen_gap},
    \begin{equation*}
        \Big|\mathbb{E}\Big[\frac{n}{N(\Tilde{X}^{\rm te}, X_2,\dots,X_n, X)}\Big] - \frac{n}{2\mathbb{E}[\Tilde{\omega}^{\rm te}] + (n-1)}\Big| \leq \Big|\frac{n}{2} \Lambda_n {\rm Var}\big[N(\Tilde{X}^{\rm te}, X_2,\dots,X_n, X)\big]\Big|,
    \end{equation*}
    where $\Lambda_n > 0$ is the smallest constant such that for a given $n$, then $\mathbb{P}\big(\frac{2}{N^3(\Tilde{X}^{\rm te}, X_2,\dots,X_n, X)} < \Lambda_n\big) = 1$. As $\omega(X_i)\in[A,B]$ almost surely, $\Lambda_n = O(1/n^3)$, see Lemma~\ref{lem:Lambda_order}. Further, from Lemma~\ref{lem:bounded_variance_N}, ${\rm Var}\big[N(\Tilde{X}^{\rm te}, X_2,\dots,X_n, X)\big] = O(n)$. Hence, Jensen's gap is upper bounded by a term that converges to zero as $n\to\infty$. This implies that $\mathbb{E}\big[n/N(\Tilde{X}^{\rm te}, X_2,\dots,X_n, X)\big] \to 1$ as $n \to \infty$.
    Repeating the same arguments
    \begin{align*}
        \mathbb{E}\Big[\frac{n^2}{N^2(\Tilde{X}^{\rm te}, \breve{X}^{\rm te}, X_3,\dots,X_n, X)}\Big] &\geq \frac{n^2}{\mathbb{E}\big[N(\Tilde{X}^{\rm te}, \breve{X}^{\rm te}, X_3,\dots,X_n, X)\big]^2}\\
        &= \frac{n^2}{\big(3\mathbb{E}[\Tilde{\omega}^{\rm te}] + (n-2)\mathbb{E}[\omega_2]\big)^2}\\
        &= \frac{n^2}{\big(3\mathbb{E}[\Tilde{\omega}^{\rm te}] + (n-2)\big)^2}.
    \end{align*}
    By the almost surely boundedness of the weight function, $\mathbb{E}[\Tilde{\omega}^{\rm te}]$ is bounded and hence $n^2/(3\mathbb{E}[\Tilde{\omega}^{\rm te}] + (n-2))^2\to 1$ as $n \to \infty$. Jensen's gap can be bounded by second order properties
    \begin{equation*}
        \begin{split}
        &\bigg|\mathbb{E}\Big[\frac{n^2}{N^2(\Tilde{X}^{\rm te}, \breve{X}^{\rm te}, X_3,\dots,X_n, X)}\Big] - \frac{n^2}{\big(3\mathbb{E}[\Tilde{\omega}^{\rm te}] + (n-2)\big)^2}\bigg| \\
        &\qquad\qquad\qquad\qquad\qquad\qquad\leq \bigg|\frac{n^2}{2} \Lambda_n' {\rm Var}\big[N(\Tilde{X}^{\rm te}, \breve{X}^{\rm te}, X_3,\dots,X_n, X)\big]\bigg|,
        \end{split}
    \end{equation*}
    where $\Lambda_n' > 0$ is the smallest constant such that for a given $n$, then $\mathbb{P}\big(\frac{6}{N^4(\Tilde{X}^{\rm te}, \breve{X}^{\rm te}, X_3^{\rm tr},\dots,X_n^{\rm tr}, X)} < \Lambda_n\big) = 1$. By regularity of the weight function and using Lemma~\ref{lem:Lambda_order}, $\Lambda_n = O(1/n^4)$, and additionally, using Lemma~\ref{lem:bounded_variance_N}, ${\rm Var}\big[N(\Tilde{X}^{\rm te}, \breve{X}^{\rm te}, X_3,\dots,X_n, X)\big] = O(n)$. Hence, Jensen's gap is upper bounded by a term that converges to zero as $n\to\infty$. This implies that $\mathbb{E}\big[n^2/N^2(\Tilde{X}^{\rm te}, \breve{X}^{\rm te}, X_3,\dots,X_n, X)\big] \to 1$ as $n \to \infty$.

    The remaining terms in Eq.~\eqref{eq:expanded_square} are quickly observed to be of order $\mathbb{E}[1/N]$, $\mathbb{E}[1/M]$, $\mathbb{E}[1/N^2]$ or $\mathbb{E}[1/M^2]$, and by almost surely boundedness of the weights, each of these terms are either of order $1/n$ or $1/n^2$.

    Combining the bound on the expected value and the tail bound implies the result.
\end{proof}

\section{Designing the weights}\label{sec:design_weights}
The methodology presented in Section~\ref{sec:method} relies on two levels of weights: the local weights $\hat{\omega}^k$ meant to approximate the density ratio ${\rm d}\mathbb{P}_X/{\rm d}\mathbb{P}_X^{k}$, and the aggregation weights $w_k$ meant to measure distribution similarity.

\paragraph{Density ratio weights:} Density ratios play a central role in the context of distribution shifts in machine learning \citep{Huang2006:Correcting,Sugiyama2007:Covariate,Sugiyama2008:Direct,Gretton2009:Dataset,Tibshirani2019:Conformal}, and estimation of density ratios is in itself a large and active research field \citep{Nguyen2007:Estimating,Sugiyama2012:Density,Choi2022:Density,Yu2025:Density}.

In this paper, we use the method also applied by \cite{Tibshirani2019:Conformal}. The idea is that
\begin{equation*}
    \frac{\mathbb{P}(X\sim\mathbb{P}_X|X=x)}{\mathbb{P}(X\sim\mathbb{P}_X^k|X=x)} = \frac{\mathbb{P}(X\sim\mathbb{P}_X)}{\mathbb{P}(X\sim\mathbb{P}_X^k)}\frac{{\rm d}\mathbb{P}_X}{{\rm d}\mathbb{P}_X^k}(x),
\end{equation*}
hence, the weight $\mathbb{P}(X\sim\mathbb{P}_X|X=x)/\mathbb{P}(X\sim\mathbb{P}_X^k|X=x)$ is proportional to the density ratio $[{\rm d}\mathbb{P}_X/{\rm d}\mathbb{P}_X^k](x)$. A model trained to solve the binary classification problem of determining if an input covariate $X$ is sampled from $\mathbb{P}_X$ or $\mathbb{P}_X^k$, with soft output scores $\hat{g}^k:\mathcal{X}\to(0,1)$, exactly provides an estimate of $\mathbb{P}(X\sim\mathbb{P}_X|X=x)$. Such a model can be fitted to the training data $\{(\Tilde{X}_i^k, 0)\}_{i=1}^{\Tilde{n}_k}\cup\{(\Tilde{X}_i^1, 1)\}_{i=1}^{\Tilde{n}_1}$. Hence, we may estimate the density ratio weights (up to proportionality) by $\hat{\omega}^k(X) = \hat{g}^k(X)/(1-\hat{g}^k(X))$. Employing this approach allows for a highly flexible solution: depending on the data, the practitioner may select a suitable model for the binary classification problem. In this work, we limited our attention to fully connected neural networks, since this provides a powerful approximator well-suited for a wide variety of datasets.

If each of the $K$ participating \glspl{da} require calibration, the method described above in fact requires $K(K-1)/2$ distinct binary classifiers, one for each of the pair-wise density ratios, $[{\rm d}\mathbb{P}_X^k/{\rm d}\mathbb{P}_X^{k'}](x)$ for $k,k'\in[K]$, $k\neq k'$. There is another possible solution which only requires a total of $K$ binary classifiers. The idea is that
\begin{equation*}
    \frac{\mathbb{P}(X\sim\mathbb{P}_X^k|X=x)}{\mathbb{P}(X\sim\mathbb{P}_X^{k'}|X=x)} = \frac{\mathbb{P}(X\sim\mathbb{P}_X^k|X=x)}{\mathbb{P}(X\sim\sum_{j=1}^{K}\mathbb{P}_X^j|X=x)}\frac{\mathbb{P}(X\sim\sum_{j=1}^{K}\mathbb{P}_X^j|X=x)}{\mathbb{P}(X\sim\mathbb{P}_X^{k'}|X=x)}.
\end{equation*}
The density ratio $\mathbb{P}(X\sim\mathbb{P}_X^k|X=x)/\mathbb{P}(X\sim\sum_{j=1}^{K}\mathbb{P}_X^j|X=x)$ can be estimated by training a binary classifier on $\big(\{(\Tilde{X}_i^k, 1)\}_{i=1}^{\ell} \cup \{(\Tilde{X}_i^k, 0)\}_{i=\ell+1}^{\Tilde{n}_k}\big)\cup\big(\bigcup_{k'\neq k} \{(\Tilde{X}_i^{k'}, 0)\}_{i=1}^{\Tilde{n}_{k'}}\big)$ with appropriate training data weighting, for $\ell \in [\Tilde{n}_k]$, yielding soft output scores $\hat{\varrho}^k : \mathcal{X} \to (0,1)$, thereby designing the density ratio weights (up to proportionality) by 
\begin{equation*}
    \hat{\omega}^k(X) = \frac{\hat{\varrho}^1(X)(1-\hat{\varrho}^k(X))}{\hat{\varrho}^k(X)(1-\hat{\varrho}^1(X))}.
\end{equation*}

Another potential solution is multi-class classification. In such a setting, a model is trained on data, $\bigcup_{k=1}^K \{(\Tilde{X}_i^k, e_k)\}_{i=1}^{\Tilde{n}_k}$ where $e_k$ is the vector with $0$ elements except an entry of $1$ in the $k$-th entry, yielding soft output scores $\hat{\rho} : \mathcal{X} \to (0,1)^{K}$ with $\sum_{k=1}^K \hat{\rho}_k = 1$. Here, $\hat{\rho}_k$ is an estimate of $\mathbb{P}(X\sim\mathbb{P}_X^k|X=x)$, and the density ratios are estimated as $\hat{\omega}^k = \hat{\rho}_1/\hat{\rho}_k$ for $k\in\{2,\dots,K\}$. The downside of this approach is that learning a multi-class classifier is more challenging than learning a binary classifier.

The methods described above falls into the category of probabilistic classification approaches. Other methods include moment matching \citep{Huang2006:Correcting}, and minimization of distribution divergences \citep{Nguyen2007:Estimating}, see also \cite{Sugiyama2012:Density} for a detailed review.

\paragraph{Distribution similarity weights:} The aggregation weights, $w_k$, can be used to provide more weight to calibration information from \glspl{da} with a higher distribution similarity. In this paper, for simplicity and to avoid additional communication overhead, we used $w_k \propto n_{\rm eff}^k$. However, a plethora of other options could be investigated, including distribution divergence measures such as the Kullback-Leibler divergence, distribution distance metrics such as the Wasserstein metric, relative divergence measures \citep{Yamada2011:Relative}, and outlier distribution factors \citep{Perini2023:Estimating,Vejling2026:Conformal}.

Consider as an example the Kullback-Leibler divergence
\begin{equation*}
    {\rm KL}\big(\mathbb{P}_X, \mathbb{P}_X^k\big) = \int {\rm log}\big([{\rm d}\mathbb{P}_X/{\rm d}\mathbb{P}_X^k](X)\big) {\rm d}\mathbb{P}_X(X).
\end{equation*}
Computing the Kullback-Leibler divergence is non-trivial in practice and also has the downside of being non-symmetric in its arguments. Further, convergence rates of estimators is slow, and a more reliable solution is relative divergence measures proposed by \cite{Yamada2011:Relative}. Here, the divergence between $\mathbb{P}_X$ and $(1-\epsilon)\mathbb{P}_X+\epsilon\mathbb{P}_X^k$ for $0 < \epsilon \leq 1$ is considered, and this provably improves convergence rates \citep{Yamada2011:Relative}. Another approach is based on outlier distribution factors. Here, the distribution $\mathbb{P}_X^k$ is expressed as a mixture between $\mathbb{P}_X$ and a proper outlier distribution $\mathbb{Q}_X^k$, see \cite{Blanchard2010:Semi}, specifically, $\mathbb{P}_X^k = (1-\pi) \mathbb{P}_X + \pi\mathbb{Q}_X^k$ for an outlier factor $0 < \pi \leq 1$. The outlier factor directly provides a measure of distribution similarity and there exists practical methods for estimating it \citep{Perini2023:Estimating,Vejling2026:Conformal}.

\section{Additional data and implementation details}\label{sec:additional_details}

\subsection{Details regarding datasets}\label{subsec:datasets}

\paragraph{Synthetic Gaussian data:} This synthetic data is generated by sampling covariates $X_i^k \stackrel{iid}{\sim} \mathcal{N}(\gamma_k, 1)$ for the $k$-th data agent where $\gamma_k \in \mathbb{R}^d$ and $d=10$ is the dimensionality. The response variable is defined as $Y_i^k = \sum_{l=1}^d [X_i^k]_l$. We evaluate three setting with varying degrees of data heterogeneity. In the severe case, data heterogeneity is simulated by sampling $[\gamma_1]_l \stackrel{iid}{\sim} {\rm Uniform}([-1,-0.5])\mathbbm{1}[l\leq \lfloor d/4\rfloor]$ for $l\in[d]$ and $[\gamma_k]_l \stackrel{iid}{\sim} {\rm Uniform}([-0.7, 0.7])\mathbbm{1}[l > \lfloor d/4\rfloor]$ for $k\in\{2,\dots,K\}$ and $l \in [d]$. In the moderate case, data heterogeneity is simulated by sampling $[\gamma_1]_l \stackrel{iid}{\sim} {\rm Uniform}([-0.5,-0.25])\mathbbm{1}[l\leq \lfloor d/4\rfloor]$ for $l\in[d]$ and $[\gamma_k]_l \stackrel{iid}{\sim} {\rm Uniform}([-0.35, 0.35])\mathbbm{1}[l > \lfloor d/4\rfloor]$ for $k\in\{2,\dots,K\}$ and $l \in [d]$. In the mild case, data heterogeneity is simulated by sampling $[\gamma_1]_l \stackrel{iid}{\sim} {\rm Uniform}([-0.25,-0.125])\mathbbm{1}[l\leq \lfloor d/4\rfloor]$ for $l\in[d]$ and $[\gamma_k]_l \stackrel{iid}{\sim} {\rm Uniform}([-0.175, 0.175])\mathbbm{1}[l > \lfloor d/4\rfloor]$ for $k\in\{2,\dots,K\}$ and $l \in [d]$. The density ratios $[{\rm d}\mathbb{P}/{\rm d}\mathbb{P}^k](X)$ are known as ratios of Gaussian densities, which allows for the use of oracle weights in the numerical experiments. The severe case is the one studied in Section~\ref{sec:numerical}.

\paragraph{Synthetic Poisson data:} This synthetic data is generated by sampling covariates $X_i^k \stackrel{iid}{\sim} \mathcal{N}(3+\gamma_k, 1)$ for the $k$-th data agent where $\gamma_k \in \mathbb{R}^d$ and $d=10$ is the dimensionality. In line with \citep{Humbert2023:One}, the response variable is defined as
\begin{equation*}
    Y_i^k | X_i^k \stackrel{iid}{\sim} {\rm Poisson}({\rm sin}^2(\bar{X}_i^k)+0.1) + 0.03 \bar{X}_i^k \mathcal{N}(0, 1) + \mathbbm{1}[{\rm Uniform}([0,1]) < 0.01] \mathcal{N}(0, 5),
\end{equation*}
where $\bar{X}_i^k = \sum_{l=1}^d [X_i^k]_l$. Data heterogeneity is simulated in the same way as for the synthetic Gaussian data.

\paragraph{Airfoil data \citep{Brooks1989:Airfoil}:} This NASA dataset arose from studying different sizes of NACA $0012$ airfoils under various conditions. It consists of $1503$ instances of $d=5$ covariates and $1$ response. The covariates are frequency in Hertz, angle of attack in degrees, chord length in meters, free-stream velocity in meters per second, and suction side displacement thickness in meters. The covariates are normalized using the min-max rule. The response is the scaled sound pressure level in decibels. The frequency and suction side displacement thickness are log-transformed in our numerical experiments.

Data heterogeneity is simulated using \emph{exponential tilting} as in \cite{Tibshirani2019:Conformal}. This is done as it allows the oracle baseline to know the true density ratio weights up to proportionality. For each data agent $k \in [K]$, data points are sampled with replacement with probabilities proportional to ${\rm exp}(X^\top \zeta^k)$ where $\zeta^k \in \mathbb{R}^d$ are tilting arrays. By construction, $[{\rm d}\mathbb{P}/{\rm d}\mathbb{P}^k](X) \propto {\rm exp}(X^\top (\zeta^1-\zeta^k))$. We evaluate three setting with varying degrees of data heterogeneity. In the severe case, we generate $[\zeta^1]_l \stackrel{iid}{\sim} {\rm Uniform}([1,3])\mathbbm{1}[l\leq \lfloor d/4\rfloor]$ for $l\in[d]$ and $[\zeta^k]_l \stackrel{iid}{\sim} {\rm Uniform}([0,2]) \mathbbm{1}[l>\lfloor d/4\rfloor]$ for $k \in \{2,\dots,K\}$ and $l \in [d]$. In the moderate case, we generate $[\zeta^1]_l \stackrel{iid}{\sim} {\rm Uniform}([0.5,1.5])\mathbbm{1}[l\leq \lfloor d/4\rfloor]$ for $l\in[d]$ and $[\zeta^k]_l \stackrel{iid}{\sim} {\rm Uniform}([0,1]) \mathbbm{1}[l>\lfloor d/4\rfloor]$ for $k \in \{2,\dots,K\}$ and $l \in [d]$. In the mild case, we generate $[\zeta^1]_l \stackrel{iid}{\sim} {\rm Uniform}([0.25,0.75])\mathbbm{1}[l\leq \lfloor d/4\rfloor]$ for $l\in[d]$ and $[\zeta^k]_l \stackrel{iid}{\sim} {\rm Uniform}([0,0.5]) \mathbbm{1}[l>\lfloor d/4\rfloor]$ for $k \in \{2,\dots,K\}$ and $l \in [d]$. The severe case is the one studied in Section~\ref{sec:numerical}.

\paragraph{Bike data \citep{Fanaee2013:Event}:} With this dataset we are predicting hourly bike sharing renting counts in Washington, D.C., and the data contains $17379$ instances with $d=9$ covariates and $1$ response. The covariates are the hour of the day, whether it is a holiday (if it is a holiday it is 1, otherwise it is 0), the day of the week, whether it is a working day (if day is neither weekend nor holiday it is 1, otherwise it is 0), the weather (1 means clear sky or few clouds, 2 means cloudy, 3 means light precipitation, 4 means heavy precipitation), normalized temperature (Celsius divided by 41), normalized feeling temperature (Celsius divided by 50), relative humidity, and normalized wind speed (meters per second divided by 67). The covariates are normalized using the min-max rule. The response is the hourly bike sharing renting counts. Data heterogeneity is simulated in the same way as for the airfoil data.

\paragraph{Concrete data \citep{Yeh1998:Modeling}:} This concrete compressive strength dataset consists of $1030$ instances with $d=8$ covariates and $1$ response. The covariates are concentrations in kilograms per cubic meter of cement, blast furnace slag, fly ash, water, superplasticizer, coarse aggregate, fine aggregate, and age given in days. The covariates are normalized using the min-max rule. The response is the concrete compressive strength measured in megapascals. Data heterogeneity is simulated in the same way as for the airfoil data.

\paragraph{Crime data \citep{Redmond2009:Communities}:} This dataset concerns the prediction of violent crimes across communities in the US containing $1994$ instances with $d=41$ covariates and $1$ response.
The covariates cover aspects of demographic composition, socioeconomic status, labor market, and law enforcement presence. The covariates are normalized using the min-max rule. The response is the total number of violent crimes per $10^5$ population. Data heterogeneity is simulated in the same way as for the airfoil data.

\paragraph{Star data \citep{Achilles2008:Tennessee}:} This arose from studying how pupil-teacher ratios impact student performance in Tennessee. The dataset contains $3754$ instances with $d=20$ covariates and $1$ response.
The covariates cover aspects of student background and demographics, teacher characteristics, classroom and treatment assignment, and school-level attributes. The covariates are normalized using the min-max rule. The response is the achievement score. Data heterogeneity is simulated in the same way as for the airfoil data.

\paragraph{Protein data \citep{Rana2013:Protein}:} This dataset deals with physicochemical properties of protein tertiary structure and consists of $45730$ instances with $d=9$ covariates and $1$ response. The covariates are the total surface area, the non-polar exposed area, the fractional area of exposed non-polar residue, the fractional area of exposed non-polar part of the residue, the molecular mass weighted exposed area, the average deviation from the standard exposed area of residue, the Euclidean distance, the secondary structure penalty, and the spatial distribution constraints. The response is the size of the residue. Data heterogeneity is simulated in the same way as for the airfoil data.

\paragraph{CIFAR-10 data \citep{Hendrycks2018:Benchmarking}:} The CIFAR-10 data is a well-studied image classification dataset of $60000$ total instances with labels airplane, automobile, bird, cat, deer, dog, frog, horse, ship, and truck with images in resolution $d=32\times32\times3$. The CIFAR-10C dataset consists of the CIFAR-10 data contaminated with various types of noise of varying levels of severity (from 0-4). We restrict our attention to the brightness contamination, and evaluate three setting with varying degrees of data heterogeneity. In the severe case, we simulate data heterogeneity by sampling images random uniformly and contaminating the images with severity level $s \in \{0,1,2,3,4\}$ with probability proportional to $[\eta^k]_s \stackrel{iid}{\sim} {\rm Uniform}([0, 1])$ for $k \in [K]$. In the moderate case, we simulate data heterogeneity by sampling images random uniformly and contaminating the images with severity level $s \in \{0,1,2\}$ with probability proportional to $[\eta^k]_s \stackrel{iid}{\sim} {\rm Uniform}([0, 1])$ for $k \in [K]$. In the mild case, we simulate data heterogeneity by sampling images random uniformly and contaminating the images with severity level $s \in \{0,1\}$ with probability proportional to $[\eta^k]_s \stackrel{iid}{\sim} {\rm Uniform}([0, 1])$ for $k \in [K]$. The severe case is the one studied in Section~\ref{sec:numerical}.



\subsection{Details of benchmarks}\label{subsec:benchmarks}

\paragraph{CP:} Standard \gls{cp} calibrates using local scores $\hat{s}_1^1,\dots,\hat{s}_{n_1}^1$. To achieve \gls{mc} at level $\alpha$, the local quantile $Q_\alpha = {\rm Quantile}\big(1-\alpha; \frac{1}{n_1+1}\big(\sum_{i=1}^{n_1} \delta_{\hat{s}_i^1} + \delta_\infty\big)\big)$ is computed and the prediction set is constructed as $\mathcal{C}_\alpha(X) = \{Y \in \mathcal{Y} : \hat{s}(X,Y) \leq Q_\alpha\}$. To achieve \gls{ccc} given levels $\alpha$ and $\delta$, $\alpha^*$ is determined as the largest $\alpha'$ such that $F_{\rm Beta}(1-\alpha; (1-\alpha')(n_1+1), \alpha'(n_1+1)) \leq \delta$ where $F_{\rm Beta}$ is the \gls{cdf} of the Beta distribution. The local quantile $Q_{\alpha^*}$ is then evaluated, and the prediction set is constructed as $\mathcal{C}_{\alpha^*}(X)$.



\paragraph{FGLCP \citep{Min2025:Personalized}:} This baseline corresponds to \emph{PFCP2} of \cite{Min2025:Personalized}. First, the dataset $\mathcal{D}^1$ is split into two disjoint datasets $\mathcal{D}^1_{\rm eng}$ and $\mathcal{D}^1_{\rm cal}$, where $\mathcal{D}^1_{\rm eng}$ will be used for the engression task and $\mathcal{D}^1_{\rm cal}$ is used for the model calibration. The method then uses datasets $\mathcal{D}^1_{\rm eng}$ and $\mathcal{D}^k$ for $k\in\{2,\dots,K\}$ to approximate the conditional distribution $F(\hat{s}(X,Y)|X)$, yielding localized scores $\hat{s}_{\rm loc}(X, Y) = \hat{F}(\hat{s}(X, Y)|X)$ for $(X, Y) \in \mathcal{D}^1_{\rm cal}$. The conformal prediction set is then constructed as $\mathcal{C}_\alpha^{\rm glcp} = \{Y \in \mathcal{Y} : \hat{s}_{\rm loc}(X, Y) \leq Q_\alpha^{\rm glcp}\}$ where
\begin{equation*}
    Q_\alpha^{\rm glcp} = {\rm Quantile}\Big(1-\alpha; \frac{1}{|\mathcal{D}^1_{\rm cal}|+1} \Big(\sum_{(X,Y)\in\mathcal{D}^1_{\rm cal}} \delta_{\hat{s}_{\rm loc}(X,Y)} + \delta_1\Big)\Big).
\end{equation*}
To compute $\hat{F}(\hat{s}(X, Y)|X)$, a federated engression algorithm is employed based on a fully connected neural network with two hidden layers each with 50 hidden units followed by Platt calibration. The engression data is randomly split in a 9:1 ratio for classifier training and Platt calibration. The reader is referred \cite{Min2025:Personalized} for further details on the federated engression algorithm used to compute $\hat{F}(\hat{s}(X, Y)|X)$.

\paragraph{FCP:} Standard \gls{fcp} aggregates all the scores $\hat{s}_i^k$ for $k\in[K]$ and $i\in[n_k]$. To achieve \gls{mc} at level $\alpha$, the quantile $Q_\alpha = {\rm Quantile}\big(1-\alpha: \frac{1}{\sum_{k=1}^K n_k+1}\big(\sum_{k=1}^K\sum_{i=1}^{n_k} \delta_{\hat{s}_i^k} + \delta_\infty\big)\big)$ is computed and the prediction set is constructed as $\mathcal{C}_\alpha(X) = \{Y \in \mathcal{Y} : \hat{s}(X,Y) \leq Q_\alpha\}$. To achieve \gls{ccc} given levels $\alpha$ and $\delta$, $\alpha^*$ is determined as the largest $\alpha'$ such that
\begin{equation*}
    F_{\rm Beta}\Big(1-\alpha; (1-\alpha')\Big(\sum_{k=1}^K n_k+1\Big), \alpha'\Big(\sum_{k=1}^K n_k+1\Big)\Big) \leq \delta,
\end{equation*}
where $F_{\rm Beta}$ is the \gls{cdf} of the Beta distribution. The local quantile $Q_{\alpha^*}$ is then evaluated, and the prediction set is constructed as $\mathcal{C}_{\alpha^*}(X)$.

\paragraph{FCP-QQ \citep{Humbert2023:One,Humbert2024:Marginal}:} For \gls{mc} guarantee at level $\alpha$, the quantile-of-quantiles method finds quantiles $\bar{\beta}$ and $\bar{\tau}$ by solving
\begin{equation*}
    (\bar{\beta}, \bar{\tau}) = \argmin_{(\beta,\tau)\in(0,1)^2}~~\{\mathbb{E}[U_\beta^{(\lceil K(1-\tau)\rceil)}] : \mathbb{E}[U_\beta^{(\lceil K(1-\tau)\rceil)}] \geq 1-\alpha\},
\end{equation*}
where $U_\beta^k \sim {\rm Beta}((1-\beta)(n_k+1), \beta(n_k+1))$, and $U_\beta^{(t)}$ is the $t$-th smallest among $U_\beta^1,\dots,U_\beta^K$. This is solvable using that $U_\beta^{(\lceil K(1-\tau)\rceil)}$ follows a Beta-Beta distribution. Given the quantiles $\bar{\beta}$ and $\bar{\tau}$, the quantile-of-quantiles $Q_{\bar{\beta},\bar{\tau}} = {\rm Quantile}\big(1-\bar{\tau}; \frac{1}{K}\sum_{k=1}^K \delta_{Q_{\bar{\beta}}^k}\big)$ where $Q_{\bar{\beta}}^k = {\rm Quantile}\big(1-\bar{\beta}; \frac{1}{n_k+1}\big(\sum_{j=1}^{n_k} \delta_{\hat{s}_j^k} + \delta_\infty \big)\big)$ is computed and the prediction set is constructed as $\mathcal{C}_{\bar{\beta},\bar{\tau}}(X) = \{Y \in \mathcal{Y} : \hat{s}(X,Y) \leq Q_{\bar{\beta},\bar{\tau}}\}$. Likewise, the \gls{ccc} guarantee with levels $\alpha$ and $\delta$ can be satisfied by finding quantiles $\bar{\beta}_{\rm c}$ and $\bar{\tau}_{\rm c}$ minimizing the \gls{ccc} while maintaining
\begin{equation*}
    \mathbb{P}_\mathcal{D}\big(\mathbb{P}(Y \in \mathcal{C}_{\beta_c, \tau_c}(X) | \mathcal{D}) \geq 1 - \alpha \big) \geq 1-\delta.
\end{equation*}
This can practically be realized by exploiting that the left-hand side of the above inequality is equal to $\mathbb{P}\big(U_{\beta_c}^{(\lceil K(1-\tau_c)\rceil)} \geq 1-\alpha\big)$ and again using that $U_{\beta_c}^{(\lceil K(1-\tau_c)\rceil)}$ follows a Beta-Beta distribution.

\paragraph{FWCP \citep{Plassier2023:Conformal,Plassier2024:Efficient}:} The \gls{fwcp} approach uses all the scores $\hat{s}_i^k$ for $k\in[K]$ and $i\in[n_k]$ to determine the quantile of the score distribution weighted by an estimate, $\hat{\omega}'$, of $\omega'(X) = [{\rm d}\mathbb{P}_X/\sum_{k=2}^K {\rm d}\mathbb{P}_X^k/(K-1)](X)$. The density ratio is estimated by fitting a binary classifier to data $\{(\Tilde{X}_i^1, 1)\}_{i=1}^{\Tilde{n}_1}\cup\big(\bigcup_{k=2}^{K}\{(\Tilde{X}_i^k, 0)\}_{i=1}^{\Tilde{n}_k}\big)$. As the binary classifier we employ a fully connected neural network with a single hidden layer using the standard implementation in \texttt{scikit-learn}. To achieve \gls{mc} at level $\alpha$, the quantile
\begin{equation*}
    Q_\alpha = {\rm Quantile}\bigg(1-\alpha: \frac{\sum_{k=1}^K\sum_{i=1}^{n_k} \hat{\omega}'(X_i^k)\delta_{\hat{s}_i^k} + \hat{\omega}'(X)\delta_\infty}{\sum_{k=1}^K \sum_{i=1}^{n_k} \hat{\omega}'(X_i^k) + \hat{\omega}'(X)}\bigg),
\end{equation*}
is computed and the prediction set is constructed as $\mathcal{C}_\alpha(X) = \{Y \in \mathcal{Y} : \hat{s}(X,Y) \leq Q_\alpha\}$. To achieve \gls{ccc} given levels $\alpha$ and $\delta$, $\alpha^*$ is determined as the largest $\alpha'$ such that 
\begin{equation*}
    F_{\rm Beta}\Big(1-\alpha; (1-\alpha')\Big(\sum_{k=1}^K n_k+1\Big), \alpha'\Big(\sum_{k=1}^K n_k+1\Big)\Big) \leq \delta,
\end{equation*}
where $F_{\rm Beta}$ is the \gls{cdf} of the Beta distribution. The local quantile $Q_{\alpha^*}$ is then evaluated, and the prediction set is constructed as $\mathcal{C}_{\alpha^*}(X)$.

\paragraph{FWCP-QQ (ours):} A variant of FWCP using a quantile-of-quantiles approach is also included as a benchmark. The quantiles $\bar{\beta}$ and $\bar{\tau}$ satisfying the \gls{mc} condition are found as in the \emph{FCP-QQ} approach. Given the quantiles $\bar{\beta}$ and $\bar{\tau}$, compute $Q_{\bar{\beta},\bar{\tau}} = {\rm Quantile}\big(1-\bar{\tau}: \frac{1}{K}\sum_{k=1}^K \delta_{Q_{\bar{\beta}}^k}\big)$ where $Q_{\bar{\beta}}^k$ is the weighted quantile
\begin{equation*}
    Q_{\bar{\beta}}^k = {\rm Quantile}\bigg(1-\bar{\beta}: \frac{\sum_{j=1}^{n_k} \hat{\omega}_j'\delta_{\hat{s}_j^k} + \hat{\omega}'(X)\delta_\infty}{\bar{\omega}' + \hat{\omega}'(X)}\bigg),
\end{equation*}
using the same weights as for \emph{FWCP}.
The prediction set is constructed as $\mathcal{C}_{\bar{\beta},\bar{\tau}}(X) = \{Y \in \mathcal{Y} : \hat{s}(X,Y) \leq Q_{\bar{\beta},\bar{\tau}}\}$. Likewise, the \gls{ccc} guarantee with levels $\alpha$ and $\delta$ can be satisfied by finding quantiles $\bar{\beta}_{\rm c}$ and $\bar{\tau}_{\rm c}$ as in \emph{FCP-QQ}.

\paragraph{PFWCP (ours):}
This is the core methodology presented in Section~\ref{sec:method}. 



\paragraph{osPFWCP (ours):}
This is the communication efficient one-shot method detailed in Section~\ref{sec:method}.



\paragraph{osPFWGLCP (ours):}
This benchmark is a variant of the \emph{osPFWCP} in which federated engression is used to localize the score function as in the \emph{FGLCP} benchmark proposed by \cite{Min2025:Personalized}. Hence, this method combines the benefits of score function localization of \cite{Min2025:Personalized} with the federation of calibration information with personalized statistical validity of this work. The baseline divides the datasets $\mathcal{D}^k$ into two disjoint datasets $\mathcal{D}^k_{\rm eng}$ and $\mathcal{D}^k_{\rm cal}$ with a 1:1 ratio, where $\mathcal{D}^k_{\rm eng}$ will be used for the engression task and $\mathcal{D}^k_{\rm cal}$ is used for the model calibration. The method then uses datasets $\mathcal{D}^k_{\rm eng}$ for $k\in[K]$ to approximate the conditional distribution $F(\hat{s}(X,Y)|X)$, yielding localized scores $\hat{s}_{\rm loc}(X, Y) = \hat{F}(\hat{s}(X, Y)|X)$ for $(X, Y) \in \bigcup_{k=1}^K\mathcal{D}^k_{\rm cal}$. The conformal prediction set is then constructed using the weighted-quantile-of-quantiles approach with one-shot communication as presented in Section~\ref{sec:method}, using $\hat{s}_{\rm loc}$ in place of $\hat{s}$. As in the \emph{FGLCP} benchmark, $\hat{F}(\hat{s}(X, Y)|X)$ is computed using a federated engression algorithm based on a fully connected neural network with two hidden layers each with 50 hidden units followed by Platt calibration. The engression data is randomly split in a 9:1 ratio for classifier training and Platt calibration.



\subsection{Additional implementation details}\label{subsec:implementation}
\paragraph{Prediction model specifications and training:}
The prediction model is given by a fully connected neural network trained with federated learning on the training data of all the participating data agents. For the synthetic datasets, two hidden layers each with $30$ hidden units is used. For the real regression datasets, three hidden layers are used with $500$, $200$, and $100$ hidden units, respectively. For the \emph{cifar10} data, three hidden layers are used with $200$, $100$, and $50$ hidden units, respectively, acting on the principal components trained to explain at least $90~\%$ of the variance.

For the regression tasks, the hidden layer activations are \texttt{LeakyReLu}, and the neural networks are trained using the \texttt{adam} algorithm with the \texttt{mean squared error loss}. The batch size is $32$, the learning rate is $0.001$, and the number of epochs is $5000$.

For the classification task, the hidden layer activations are \texttt{ReLu} with \texttt{softmax} output activation, and the neural networks are trained using the \texttt{adam} algorithm with the \texttt{cross entropy loss} and an $L_2$ regularization term with parameter $0.0001$. The batch size is $32$, the learning rate is $0.001$, and the number of epochs is $5000$.

\paragraph{Density ratio estimation:}
For the synthetic datasets, the binary classifiers used to estimate the density ratios are specified as fully connected neural networks with a single $10$-unit hidden layer employing \texttt{ReLu} hidden layer activations with \texttt{sigmoid} output activation. For the real regression datasets, we use a single $30$-unit hidden layer. For the \emph{cifar10} dataset, we use also a fully connected neural network with a single $10$-unit hidden layer, here acting on the principal components trained to explain at least $90~\%$ of the variance. The neural networks are trained using the \texttt{adam} algorithm with the \texttt{cross entropy loss} and an $L_2$ regularization term with parameter $0.0001$. The batch size is set to $200$, the learning rate is set to $0.001$, and the number of epochs used is $600$.

\paragraph{Hyperparameter selection:}
The choice of $N_{\rm rep}$ (see Procedure~\ref{alg:procedure}) represents an important trade-off between computational complexity and the error made in solving Eq.~\eqref{eq:choosing_beta_and_tau}. The algorithm simulates data $U_{\beta,i}^{(\tau)} \simeq \mathbb{P}(Y \in \mathcal{C}^{\rm wqq}_{\beta,\tau}(X)|\mathcal{D})$ and, for \gls{mc}, uses the sample mean to approximate $\mathbb{P}(Y \in \mathcal{C}^{\rm wqq}_{\beta,\tau}(X)) = \mathbb{E}[\mathbb{P}(Y \in \mathcal{C}^{\rm wqq}_{\beta,\tau}(X)|\mathcal{D})]$. The variance of the sample mean equals ${\rm Var}[U_{\beta,i}^{(\tau)}]/N_{\rm rep}$. Since $U_{\beta,i}^{(\tau)}$ are random variables defined on $[0, 1]$, their variance is upper bounded by $1/4$, meaning that the standard deviation of the sample mean is upper bounded by $1/(2\sqrt{N_{\rm rep}})$. Practically, it is desirable to have $1/(2\sqrt{N_{\rm rep}}) \ll \alpha$, i.e., $N_{\rm rep} \gg 1/(4\alpha^2)$. For \gls{ccc} we are concerned with the variance of the empirical quantile estimator which has asymptotic expansion $\alpha(1-\alpha)/(N_{\rm rep}f^2(F^{-1}(1-\alpha)))$ where $F^{-1}$ and $f$ are the quantile and density functions of $U_{\beta}^{(\tau)}$, respectively \citep{Stuart1994:Kendall}, leading to the guideline $N_{\rm rep} \gg \alpha(1-\alpha)/(\delta^2 f^2(F^{-1}(1-\alpha)))$. Since $f$ and $F^{-1}$ are unknown prior to running the algorithm, this is not directly applicable but the expression is inverse proportional to $\delta^2$ giving some intuition still. In the numerical experiments, we use $N_{\rm rep} = 2000$ for \gls{mc} guarantees, and $N_{\rm rep} = 4000$ for \gls{ccc} guarantees.

In Procedure~\ref{alg:procedure}, a search space $\mathcal{G} \subset (0,1)\times[0,1)$ on the quantiles $(\beta, \tau)$ must be provided. In this work, we design this as $\{\alpha,\alpha+\Delta_\beta\dots,1/4\}\times\{0, \Delta_\tau,\dots, 1/2\}$ where $\Delta_\beta = (1/4 - \alpha)/21$ and $\Delta_\tau = 1/42$ are the step sizes. This ensures that $\beta^*$ is never small to the point of risking $Q_{\beta^*}^k$ being infinite, and generally, a value of $\beta$ closer to $1/2$ provides more stable prediction sets.
For the one-shot procedures, osPFWCP and osPFWGLCP, the search space is only on the $\tau$ quantile and is designed as $\{0, \Delta'_\tau, \dots,1-{\rm max}(w_1,\dots,w_K)\}$ where $\Delta'_\tau = (1 - {\rm max}(w_1,\dots,w_K))/50$, thereby ensuring that $1-\tau^*$ is always greater than ${\rm max}(w_1,\dots,w_K)$ such that a valid quantile can be determined.

\subsection{Details to reproduce Figure~\ref{fig:star_illustration_cdf}}\label{subsec:details_figure1}
To illustrate the core principle motivated by Proposition~\ref{prop:coverage_approximation}, we conducted a numerical experiment on the Tennessee's student teacher achievement ratio dataset \citep{Achilles2008:Tennessee}. With this data, we evaluated the \gls{ccc} for \gls{cp} with and without covariate shift, \gls{wcp} with weights estimated using logistic regression, as well as \gls{cp} without covariate shift using the estimated effective sample size. As in \cite{Tibshirani2019:Conformal}, we simulated a covariate shift through exponential tilting, specifically, the calibration data with covariate shift samples the dataset with replacement with probabilities proportional to ${\rm exp}(X^\top \zeta)$ where $\zeta=-[2.5,1,2,1,0.5,1.5,1,1.5,0.5,0.5,0,0.5,0.5,0,0,0.5,1.5,2.5,2,0]^\top$. The elements of this tilting array were generated random uniformly from $\{0,-0.5,-1,\dots,-3\}$. The significance level was set to $\alpha=0.1$, and the calibration dataset had $n=100$ samples. The empirical \gls{ccc} is evaluated for $1000$ Monte Carlo simulations of the calibration data with $1000$ test points for each. The result is shown in Figure~\ref{fig:star_illustration_cdf}.

\section{Additional numerical results}\label{sec:additional_results}
Here, we provide additional numerical results to supplement those presented in Section~\ref{sec:numerical}. This includes results with additional benchmarks, specifically, \emph{FCP} and the proposed methods using oracle weights (when these are known). Further, we include results with the three different levels of heterogeneity described in Section~\ref{subsec:datasets}. Tables~\ref{tab:reg_marg_severe} and \ref{tab:reg_cond_severe} shows extended versions of the Tables~\ref{tab:reg_marg_main} and \ref{tab:reg_cond_main}, i.e., the severe heterogeneity experiments. Tables~\ref{tab:reg_marg_moderate} and \ref{tab:reg_cond_moderate} show results for the case of moderate heterogeneity while Tables~\ref{tab:reg_marg_mild} and \ref{tab:reg_cond_mild} show the results for the case of mild heterogeneity. Overall, the results indicate the efficacy of the proposed \emph{PFWCP} method to efficiently provide the specified statistical guarantees. Meanwhile, the one-shot techniques \emph{osPFWCP} and \emph{osPFWGLCP} show promising performance to provide \gls{mc} guarantees but on occasion struggles to provide the specified \gls{ccc} guarantees, due to the lessened flexibility in the quantile specifications. Using \emph{osPFWGLCP} tends to provide efficient conformal prediction sets with less miscoverage than \emph{osPFWCP}. Comparatively, the \emph{FWCP-QQ} benchmark tends to be overly conservative, which is also to be expected considering that it does not adjust for the increased variance in the coverage due to the data heterogeneity. The \emph{FWCP} benchmark is competitive in terms of providing \gls{mc} guarantees, but tends to fail dramatically in satisfying \gls{ccc} guarantees. A similar remark applies to the \emph{FCP} and \emph{FCP-QQ} benchmarks, although these techniques tend to be overly liberal in terms of both \gls{mc} and \gls{ccc} guarantees in the case of severe heterogeneity.

%
\setlength{\tabcolsep}{0.3em} 
\begin{table}[H]
    \centering
    \caption{Severe heterogeneity. Marginal coverage (MC), conditional miscoverage (CMC), and efficiency (Eff.). Methods with MC above $1-\alpha=0.9$, above $0.89$, and below $0.89$ are marked by colors \textcolor{color4}{green}, \textcolor{color6}{yellow}, and \textcolor{color2}{red}, respectively.}
    \label{tab:reg_marg_severe}
    \footnotesize
\begin{tabular}{llrrrrrrrrr}
        \toprule
                      &                    & gaussian                    & poisson                     & airfoil                     & protein                     & bike                        & concrete                    & crime                       & star                        & cifar10c                     \\
        \midrule
        MC (\%)    & CP                 & \textcolor{color6}{$89.57$} & \textcolor{color6}{$89.77$} & \textcolor{color6}{$89.49$} & \textcolor{color6}{$89.48$} & \textcolor{color6}{$89.42$} & \textcolor{color6}{$89.65$} & \textcolor{color6}{$89.58$} & \textcolor{color6}{$89.65$} & \textcolor{color6}{$89.16$} \\
                      & FGLCP              & \textcolor{color4}{$90.38$} & \textcolor{color4}{$90.28$} & \textcolor{color6}{$89.80$} & \textcolor{color4}{$90.19$} & \textcolor{color4}{$90.23$} & \textcolor{color4}{$90.34$} & \textcolor{color6}{$89.99$} & \textcolor{color4}{$90.16$} & \textcolor{color6}{$89.96$} \\
                      & FCP                & \textcolor{color4}{$90.49$} & \textcolor{color6}{$89.84$} & \textcolor{color2}{$86.44$} & \textcolor{color4}{$90.09$} & \textcolor{color2}{$88.19$} & \textcolor{color2}{$84.26$} & \textcolor{color2}{$89.00$} & \textcolor{color2}{$88.70$} & \textcolor{color6}{$89.06$} \\
                      & FCP-QQ             & \textcolor{color2}{$88.70$} & \textcolor{color6}{$89.97$} & \textcolor{color2}{$86.48$} & \textcolor{color4}{$90.20$} & \textcolor{color2}{$88.20$} & \textcolor{color2}{$84.17$} & \textcolor{color6}{$89.11$} & \textcolor{color2}{$88.78$} & \textcolor{color6}{$89.40$} \\
                      & FWCP               & \textcolor{color4}{$90.22$} & \textcolor{color6}{$89.81$} & \textcolor{color4}{$92.32$} & \textcolor{color4}{$90.51$} & \textcolor{color6}{$89.57$} & \textcolor{color6}{$89.93$} & \textcolor{color6}{$89.11$} & \textcolor{color6}{$89.74$} & \textcolor{color2}{$85.23$} \\
\bottomrule
    \end{tabular}
\end{table}
\setlength{\tabcolsep}{0.3em} 
\begin{table}[H]
    \centering
    \caption{Severe heterogeneity. Calibration-conditional coverage (CCC), conditional miscoverage (CMC), and efficiency (Eff.). Methods with CCC above $1-\delta=0.9$, above $0.8$, and below $0.8$ are marked by colors \textcolor{color4}{green}, \textcolor{color6}{yellow}, and \textcolor{color2}{red}, respectively.}
    \label{tab:reg_cond_severe}
    \footnotesize
    \begin{tabular}{llrrrrrrrrr}
        \toprule
                      &                    & gaussian                     & poisson                      & airfoil                      & protein                     & bike                        & concrete                    & crime                       & star                        & cifar10c                     \\
        \midrule
        CCC (\%)   & CP                 & \textcolor{color6}{$86.60$}  & \textcolor{color6}{$87.20$}  & \textcolor{color6}{$86.00$}  & \textcolor{color6}{$86.80$} & \textcolor{color6}{$84.80$} & \textcolor{color6}{$86.40$} & \textcolor{color6}{$84.80$} & \textcolor{color6}{$86.80$} & \textcolor{color6}{$86.80$} \\
                      & FGLCP              & \textcolor{color4}{$90.40$}  & \textcolor{color4}{$93.00$}  & \textcolor{color4}{$91.20$}  & \textcolor{color4}{$90.80$} & \textcolor{color4}{$90.40$} & \textcolor{color4}{$91.00$} & \textcolor{color6}{$89.40$} & \textcolor{color4}{$90.60$} & \textcolor{color4}{$91.80$} \\
                      & FCP                & \textcolor{color4}{$100.00$} & \textcolor{color6}{$82.80$}  & \textcolor{color2}{$29.40$}  & \textcolor{color6}{$80.60$} & \textcolor{color2}{$33.00$} & \textcolor{color2}{$3.20$}  & \textcolor{color2}{$75.00$} & \textcolor{color2}{$60.20$} & \textcolor{color2}{$71.00$} \\
                      & FCP-QQ             & \textcolor{color4}{$100.00$} & \textcolor{color6}{$85.60$}  & \textcolor{color2}{$42.20$}  & \textcolor{color6}{$84.00$} & \textcolor{color2}{$42.40$} & \textcolor{color2}{$7.60$}  & \textcolor{color6}{$80.40$} & \textcolor{color2}{$67.40$} & \textcolor{color2}{$62.40$} \\
                      & FWCP               & \textcolor{color2}{$79.40$}  & \textcolor{color6}{$89.40$}  & \textcolor{color4}{$97.80$}  & \textcolor{color6}{$89.20$} & \textcolor{color2}{$77.40$} & \textcolor{color2}{$74.60$} & \textcolor{color4}{$96.80$} & \textcolor{color6}{$83.00$} & \textcolor{color2}{$65.60$} \\
\bottomrule
    \end{tabular}
\end{table}
\setlength{\tabcolsep}{0.3em} 
\begin{table}[H]
    \centering
    \caption{Moderate heterogeneity. Marginal coverage (MC), conditional miscoverage (CMC), and efficiency (Eff.). Methods with MC above $1-\alpha=0.9$, above $0.89$, and below $0.89$ are marked by colors \textcolor{color4}{green}, \textcolor{color6}{yellow}, and \textcolor{color2}{red}, respectively.}
    \label{tab:reg_marg_moderate}
    \footnotesize
    \begin{tabular}{llrrrrrrrrr}
        \toprule
                      &                    & gaussian                    & poisson                     & airfoil                     & protein                     & bike                        & concrete                    & crime                       & star                        & cifar10c                     \\
        \midrule
        MC (\%)    & CP                 & \textcolor{color6}{$89.74$} & \textcolor{color6}{$89.93$} & \textcolor{color6}{$89.51$} & \textcolor{color6}{$89.87$} & \textcolor{color6}{$89.73$} & \textcolor{color6}{$89.71$} & \textcolor{color6}{$89.74$} & \textcolor{color6}{$89.70$} & \textcolor{color6}{$89.25$} \\
                      & FGLCP              & \textcolor{color4}{$90.46$} & \textcolor{color4}{$90.31$} & \textcolor{color4}{$90.05$} & \textcolor{color4}{$90.30$} & \textcolor{color4}{$90.57$} & \textcolor{color4}{$90.39$} & \textcolor{color4}{$90.27$} & \textcolor{color4}{$90.60$} & \textcolor{color4}{$90.18$} \\
                      & FCP                & \textcolor{color2}{$88.97$} & \textcolor{color4}{$90.11$} & \textcolor{color6}{$89.86$} & \textcolor{color4}{$90.06$} & \textcolor{color6}{$89.47$} & \textcolor{color6}{$89.17$} & \textcolor{color4}{$90.17$} & \textcolor{color6}{$89.45$} & \textcolor{color6}{$89.46$} \\
                      & FCP-QQ             & \textcolor{color6}{$89.08$} & \textcolor{color4}{$90.20$} & \textcolor{color6}{$89.99$} & \textcolor{color4}{$90.23$} & \textcolor{color6}{$89.57$} & \textcolor{color6}{$89.28$} & \textcolor{color4}{$90.34$} & \textcolor{color6}{$89.58$} & \textcolor{color6}{$89.96$} \\
                      & FWCP               & \textcolor{color4}{$90.93$} & \textcolor{color4}{$90.72$} & \textcolor{color4}{$91.02$} & \textcolor{color6}{$89.97$} & \textcolor{color4}{$90.39$} & \textcolor{color4}{$90.53$} & \textcolor{color4}{$91.66$} & \textcolor{color4}{$90.08$} & \textcolor{color2}{$86.33$} \\
\bottomrule
    \end{tabular}
\end{table}
\setlength{\tabcolsep}{0.3em} 
\begin{table}[H]
    \centering
    \caption{Moderate heterogeneity. Calibration-conditional coverage (CCC), conditional miscoverage (CMC), and efficiency (Eff.). Methods with CCC above $1-\delta=0.9$, above $0.8$, and below $0.8$ are marked by colors \textcolor{color4}{green}, \textcolor{color6}{yellow}, and \textcolor{color2}{red}, respectively.}
    \label{tab:reg_cond_moderate}
    \footnotesize
    \begin{tabular}{llrrrrrrrrr}
        \toprule
                      &                    & gaussian                    & poisson                     & airfoil                     & protein                     & bike                        & concrete                    & crime                       & star                        & cifar10c                     \\
        \midrule
        CCC (\%)   & CP                 & \textcolor{color6}{$87.60$} & \textcolor{color6}{$84.20$} & \textcolor{color6}{$85.00$} & \textcolor{color6}{$84.20$} & \textcolor{color6}{$86.40$} & \textcolor{color6}{$84.60$} & \textcolor{color6}{$87.00$} & \textcolor{color6}{$83.40$} & \textcolor{color6}{$84.80$} \\
                      & FGLCP              & \textcolor{color6}{$89.20$} & \textcolor{color6}{$88.80$} & \textcolor{color4}{$91.80$} & \textcolor{color6}{$89.00$} & \textcolor{color6}{$89.80$} & \textcolor{color6}{$89.40$} & \textcolor{color4}{$92.40$} & \textcolor{color6}{$88.20$} & \textcolor{color6}{$89.60$} \\
                      & FCP                & \textcolor{color2}{$74.40$} & \textcolor{color6}{$88.40$} & \textcolor{color2}{$77.80$} & \textcolor{color2}{$76.80$} & \textcolor{color2}{$60.00$} & \textcolor{color2}{$64.20$} & \textcolor{color6}{$80.60$} & \textcolor{color2}{$71.00$} & \textcolor{color2}{$59.60$} \\
                      & FCP-QQ             & \textcolor{color6}{$80.40$} & \textcolor{color6}{$89.40$} & \textcolor{color6}{$81.60$} & \textcolor{color6}{$80.20$} & \textcolor{color2}{$67.60$} & \textcolor{color2}{$72.00$} & \textcolor{color6}{$84.60$} & \textcolor{color2}{$75.80$} & \textcolor{color2}{$48.80$} \\
                      & FWCP               & \textcolor{color2}{$59.00$} & \textcolor{color4}{$90.20$} & \textcolor{color6}{$88.40$} & \textcolor{color2}{$75.60$} & \textcolor{color6}{$80.60$} & \textcolor{color6}{$87.00$} & \textcolor{color4}{$95.60$} & \textcolor{color2}{$74.00$} & \textcolor{color2}{$41.20$} \\
\bottomrule
    \end{tabular}
\end{table}
\setlength{\tabcolsep}{0.3em} 
\begin{table}[H]
    \centering
    \caption{Mild heterogeneity. Marginal coverage (MC), conditional miscoverage (CMC), and efficiency (Eff.). Methods with MC above $1-\alpha=0.9$, above $0.89$, and below $0.89$ are marked by colors \textcolor{color4}{green}, \textcolor{color6}{yellow}, and \textcolor{color2}{red}, respectively.}
    \label{tab:reg_marg_mild}
    \footnotesize
    \begin{tabular}{llrrrrrrrrr}
        \toprule
                      &                    & gaussian                    & poisson                     & airfoil                     & protein                     & bike                        & concrete                    & crime                       & star                        & cifar10                     \\
        \midrule
        MC (\%)    & CP                 & \textcolor{color6}{$89.49$} & \textcolor{color6}{$89.59$} & \textcolor{color6}{$89.62$} & \textcolor{color6}{$89.53$} & \textcolor{color6}{$89.55$} & \textcolor{color6}{$89.60$} & \textcolor{color6}{$89.48$} & \textcolor{color6}{$89.60$} & \textcolor{color2}{$88.81$} \\
                      & FGLCP              & \textcolor{color4}{$90.30$} & \textcolor{color4}{$90.25$} & \textcolor{color4}{$90.01$} & \textcolor{color4}{$90.09$} & \textcolor{color4}{$90.19$} & \textcolor{color4}{$90.21$} & \textcolor{color4}{$90.03$} & \textcolor{color4}{$90.20$} & \textcolor{color4}{$90.39$} \\
                      & FCP                & \textcolor{color6}{$89.92$} & \textcolor{color4}{$90.37$} & \textcolor{color4}{$90.15$} & \textcolor{color4}{$90.05$} & \textcolor{color6}{$89.38$} & \textcolor{color6}{$89.46$} & \textcolor{color4}{$90.23$} & \textcolor{color4}{$90.07$} & \textcolor{color6}{$89.56$} \\
                      & FCP-QQ             & \textcolor{color4}{$90.03$} & \textcolor{color4}{$90.40$} & \textcolor{color4}{$90.31$} & \textcolor{color4}{$90.13$} & \textcolor{color6}{$89.49$} & \textcolor{color6}{$89.67$} & \textcolor{color4}{$90.39$} & \textcolor{color4}{$90.26$} & \textcolor{color4}{$90.06$} \\
                      & FWCP               & \textcolor{color2}{$88.68$} & \textcolor{color6}{$89.98$} & \textcolor{color6}{$89.90$} & \textcolor{color4}{$90.07$} & \textcolor{color4}{$90.79$} & \textcolor{color6}{$89.83$} & \textcolor{color4}{$91.00$} & \textcolor{color6}{$89.72$} & \textcolor{color2}{$86.81$} \\
\bottomrule
    \end{tabular}
\end{table}
\setlength{\tabcolsep}{0.3em} 
\begin{table}[H]
    \centering
    \caption{Mild heterogeneity. Calibration-conditional coverage (CCC), conditional miscoverage (CMC), and efficiency (Eff.). Methods with CCC above $1-\delta=0.9$, above $0.8$, and below $0.8$ are marked by colors \textcolor{color4}{green}, \textcolor{color6}{yellow}, and \textcolor{color2}{red}, respectively.}
    \label{tab:reg_cond_mild}
    \footnotesize
    \begin{tabular}{llrrrrrrrrr}
        \toprule
                      &                  & gaussian                    & poisson                     & airfoil                     & protein                     & bike                        & concrete                    & crime                       & star                        & cifar10                     \\
        \midrule
        CCC (\%)   & CP                 & \textcolor{color6}{$89.20$} & \textcolor{color6}{$87.20$} & \textcolor{color6}{$83.00$} & \textcolor{color6}{$84.40$} & \textcolor{color6}{$86.40$} & \textcolor{color6}{$85.40$} & \textcolor{color6}{$88.40$} & \textcolor{color6}{$84.60$} & \textcolor{color6}{$84.40$} \\
                      & FGLCP              & \textcolor{color4}{$91.20$} & \textcolor{color4}{$91.60$} & \textcolor{color4}{$90.60$} & \textcolor{color6}{$88.20$} & \textcolor{color6}{$89.60$} & \textcolor{color4}{$91.00$} & \textcolor{color4}{$91.60$} & \textcolor{color4}{$90.20$} & \textcolor{color6}{$89.00$} \\
                      & FCP                & \textcolor{color2}{$78.20$} & \textcolor{color6}{$85.40$} & \textcolor{color6}{$83.20$} & \textcolor{color2}{$77.80$} & \textcolor{color2}{$72.60$} & \textcolor{color2}{$71.60$} & \textcolor{color6}{$85.80$} & \textcolor{color6}{$82.00$} & \textcolor{color2}{$63.60$} \\
                      & FCP-QQ             & \textcolor{color6}{$83.20$} & \textcolor{color6}{$86.60$} & \textcolor{color6}{$85.00$} & \textcolor{color6}{$82.80$} & \textcolor{color2}{$77.60$} & \textcolor{color2}{$75.80$} & \textcolor{color6}{$88.40$} & \textcolor{color6}{$86.80$} & \textcolor{color2}{$46.40$} \\
                      & FWCP               & \textcolor{color6}{$81.00$} & \textcolor{color4}{$92.80$} & \textcolor{color6}{$85.60$} & \textcolor{color2}{$76.20$} & \textcolor{color4}{$92.40$} & \textcolor{color2}{$75.00$} & \textcolor{color4}{$93.40$} & \textcolor{color2}{$71.00$} & \textcolor{color2}{$46.80$} \\
\bottomrule
    \end{tabular}
\end{table}

\section{Use of Large Language Models}\label{sec:LLM}
Large language models have been used for the purpose of text editing.

\newpage

\newpage


\bibliography{bib.bib}

\end{document}